\documentclass[conference]{IEEEtran}
\usepackage[letterpaper,margin=0.75in,columnsep=0.25in]{geometry}
\usepackage{times}
\usepackage{amsmath,amssymb,amsfonts}
\usepackage{amsthm}
\usepackage{graphicx}
\usepackage{dblfloatfix}
\usepackage{subcaption}
\usepackage{xcolor}
\usepackage{booktabs}
\usepackage{enumitem}
\usepackage{hyperref}
\usepackage{algorithm}
\usepackage{multirow}
\usepackage{algorithmic}
\newtheorem{assumption}{Assumption}
\newtheorem{proposition}{Proposition}
\usepackage{booktabs}
\usepackage[table]{xcolor}
\usepackage{tabularx}
\usepackage{array}

\begin{document}

\title{Compression Footprints as Security Signals for Model-Poisoning Defense in Federated Learning}
\author{\IEEEauthorblockN{Sachi Shome}
	\IEEEauthorblockA{Stevens Institute of Technology\\
		sshome@stevens.edu}
	\and
	\IEEEauthorblockN{William Eiers}
	\IEEEauthorblockA{Stevens Institute of Technology\\
		weiers@stevens.edu}
	}


\maketitle

\begin{abstract}
Lossy compression is widely used in Federated Learning (FL) but is generally treated as an error source, while conventional poisoning defenses inspect update geometry.
In this work, we instead treat the compressor's response as a security signal: the input-dependent distortion and payload behavior induced by lossy compression can expose differences between honest and attack-generated updates. 
We introduce the concept of a \emph{compression footprint}: the low-dimensional collection of reconstruction, directional, sparsity, and payload statistics induced by a lossy compressor. 
We characterize sufficient conditions under which compression footprints separate honest and malicious updates, and operationalize our findings in the CRAFT (\emph{Compression-guided Robust Aggregation via Footprint Trust}) server-side robust aggregation method. 
Crucially, under a strict honest-majority assumption, CRAFT uses server-verifiable footprints, requires no client-side metadata nor knowledge of the number of malicious clients, and adds no communication beyond the compressed FL pipeline.
Moreover, while CRAFT assumes a strict honest majority, it does not require the number of malicious clients to be known in advance.
We observe that error-bounded lossy compressor (EBLC) footprints provide stronger separation than Top-K footprints and that footprint trust suppresses malicious influence.
We evaluate CRAFT under IID client data with 36\% malicious participation across six standard model-poisoning attacks, three datasets, and six robust aggregation baselines, finding that CRAFT consistently achieves the best accuracy in 7 out of 18 settings and within 1.7 percentage points of the best in the others.
Our results show that lossy compression can serve as both a communication mechanism and a security signal for robust aggregation in FL.
\end{abstract}


%
\IEEEpeerreviewmaketitle

\section{Introduction}
In Federated Learning (FL)~\cite{mcmahan2017communication}, clients collaboratively train a global model by computing local updates on their private data and sending them to a central server for aggregation. In practice, the model updates are sizeable so compression is often applied to reduce communication costs. These updates are also untrusted: Byzantine clients can submit poisoned updates to manipulate the global model. 
This means that FL deployments must be communication efficient and robust to malicious clients. 
In this work, we ask a different question: can the structure of compression-induced loss be used as a security signal to distinguish honest and malicious updates?


Conventional Byzantine defenses~\cite{yin2018byzantine,blanchard2017machine,el2018hidden,shejwalkar2021manipulating} estimate malicious behavior from update values, distances, directions, or projections. 
Compression-aware robust-learning methods~\cite{HU2026108582,Xu_2025_WACV,pmlr-v238-rammal24a,Zhu_2023} treat compression as a source of error and do not leverage the structure of the discarded information. 
Although lossy compression discards update information, the structure of that loss reveals behavior that conventional raw-update statistics may obscure. 
Our key observation is that reconstruction residuals and payload structure induced by lossy compressors, specifically error-bounded lossy compressors (EBLCs), expose attack-sensitive behavior: honest and malicious updates can produce different reconstruction errors, directional changes, near-zero behavior, and compressed sizes. 
Summarizing these differences yields a low-dimensional input-dependent \emph{compression footprint}.
Crucially, compression does not create information; rather, it transforms  updates into a representation where attack-induced structure is more separable from honest updates. 
EBLCs differ from popular lossy compression techniques like Top-$K$ sparsification~\cite{shi2019understanding}: while Top-$K$'s footprint is dominated by a single retained-versus-discarded squared-magnitude statistic, EBLCs expose several structural dimensions such as mean-squared error and sparsity change. 
Thus, lossy compression is not only a communication primitive: an update's response to lossy compression can also serve as a security signal for robust aggregation.

To characterize when this signal is useful, we establish a sufficient condition for footprint separability and formalize this in Section~\ref{sec:compression-footprint-separation}.
The intuition is that honest clients form a concentrated dominant footprint region, while sufficiently influential attacks tend to shift compressor-sensitive structure. 
This shift can induce a footprint distribution that is separable from the honest distribution when the distance between footprint centers exceeds the combined concentration radii. 
This impact-separation relationship gives the design intuition behind CRAFT: that attacks producing a larger footprint shift are more likely to be separable, while overlapping footprints receive graded rather than binary treatment. 

Several challenges arise when using compression footprints for robust aggregation. 
The server cannot trust a footprint computed during client-side compression, since malicious clients can submit poisoned updates together with misleading metadata and the server itself does not possess the original uncompressed update needed to verify that footprint.
Recompression potentially addresses the untrusted-metadata issue by computing a new diagnostic footprint, but the footprint is sensitive to the scale of the input. 
Even if the footprint is computed correctly, honest and malicious footprints may partially overlap, so clearly distinguishing malicious from honest clients may be impossible. 
Moreover, the server generally does not know the number of malicious clients in any given round. 
A footprint-utilizing defense must therefore be designed to address these challenges.

To that end, we propose CRAFT (\emph{Compression-guided Robust Aggregation via Footprint Trust}), a server-side robust aggregation method that uses compression footprints to assign continuous trust weights before aggregation. 
CRAFT differs from existing robust aggregation approaches by evaluating how an update responds to a fixed, server-controlled lossy transformation, instead of relying only on where it lies in parameter space. 
CRAFT assumes a strict honest majority and focuses on IID client data, where honest updates are expected to form a compact footprint region. CRAFT works as follows. 
In each round, clients transmit their compressed updates to the server. The server decompresses the received payloads, normalizes the resulting updates round-wise, then performs a uniform recompression pass to extract the compression footprints. 
Then, CRAFT assigns trust weights to each client based on their distance from the majority-supported footprint core, with trust decaying with distance from the core. 
Finally, CRAFT applies a trust-weighted coordinate-wise median to the decompressed updates to produce the aggregated update. 
Three properties of CRAFT are worth emphasizing. 
First, footprint extraction is server-controlled and does not rely on client-reported metadata. 
Second, recompression is diagnostic only: the final aggregation uses the received, once-decompressed updates. 
Third, CRAFT does not require knowledge of the number of malicious clients, instead relying on the majority-supported footprint core to estimate the dominant footprint region.

Our experiments show that EBLCs in particular produce a stronger honest-malicious footprint separation than Top-$K$ sparsification, that separation changes based on compressor configuration, and that CRAFT's trust-weighted aggregation suppresses malicious influence.
We evaluate CRAFT on three datasets spanning image and tabular models, against six model-poisoning attacks in a setting with 9 of 25 clients malicious, and compare it to six representative robust aggregation baselines. 
CRAFT achieves the highest final accuracy in 7 of 18 dataset-attack settings, and is within 1.7 percentage points of the best method in the others.

Our contributions are as follows:
\begin{itemize}
\item We introduce the concept of compression footprints as a low-dimensional server-verifiable signal for distinguishing honest and malicious FL updates, including analysis of why error-bounded lossy compression can induce richer and empirically more separable footprints than Top-$K$ sparsification. We characterize sufficient conditions under which compression footprints separate honest and malicious updates.
\item We propose CRAFT, a server-side robust aggregation method that combines robust normalization, server-side footprint extraction, majority-supported soft trust, and weighted-median aggregation without requiring the number of malicious clients to be known in advance.
\item We evaluate compressor choice, error tolerance, and malicious-client fraction to characterize when compression footprints separate malicious updates and suppress their influence. We evaluate CRAFT across three datasets, six model-poisoning attacks, and six aggregation baselines with 36\% malicious participation, finding that CRAFT consistently achieves the best accuracy in 7 of 18 settings and is within 1.7 percentage points of the best in the others. 
\end{itemize}

\section{Background}
\label{sec:background}

\subsection{Poisoning Attacks in Federated Learning}
\label{subsec:model-poisoning-background}

Federated learning is vulnerable to poisoning because the server updates the
global model using client-submitted information. These attacks are commonly
categorized by the adversary's goal and capability~\cite{bhagoji2019analyzing,
fang2020local,shejwalkar2021manipulating,bagdasaryan2020backdoor}. Targeted
attacks aim to induce incorrect behavior on specific inputs while preserving
benign accuracy; backdoor and model-replacement attacks are representative
examples~\cite{bagdasaryan2020backdoor,sun2019backdoor,xie2020dba,
zhang2022neurotoxin}. Untargeted attacks instead aim to reduce overall test
accuracy, and are the focus of this work.

Based on capability, poisoning attacks are either data poisoning or model
poisoning. In data poisoning, malicious clients corrupt local training data. In
model poisoning, Byzantine clients directly manipulate the updates sent to the
server, making the attack stronger because submitted updates need not correspond
to ordinary local training~\cite{bhagoji2019analyzing,fang2020local,
shejwalkar2021manipulating}.

Existing model-poisoning attacks exploit different weaknesses of aggregation.
ALIE uses coordinate-wise benign statistics to place malicious updates in a
plausible range while shifting the aggregate~\cite{baruch2019little}. IPM sends
updates in the opposite direction of the estimated benign update, pulling the
aggregate away from descent~\cite{xie2020fall}. Min-Max and Min-Sum optimize
malicious updates to evade distance-based defenses by constraining either the
maximum distance or the sum of distances to benign updates~\cite{shejwalkar2021manipulating}.
Other attacks target specific robust rules~\cite{fang2020local}, exploit
multi-round consistency~\cite{xie2025model}, or use structured update geometry
such as similarity-based manipulation~\cite{kasyap2024sine}. Overall, malicious
updates need not be obvious outliers; they may preserve plausible coordinate
values, pairwise distances, or directions while still degrading the global model.

\subsection{Existing Robust Aggregation Algorithms}
\label{subsec:robust-aggregation-background}

In non-adversarial FL, the server commonly aggregates updates using the
mean~\cite{mcmahan2017communication},
$\hat{g}^{(t)} = \frac{1}{n}\sum_{i=1}^{n} g_i^{(t)}$.
Although effective when all clients are honest, mean aggregation is highly
sensitive to Byzantine clients because every submitted update receives equal
weight.

Robust aggregation rules reduce this influence by replacing or filtering the
mean. Coordinate-wise methods, such as median and trimmed mean, aggregate each
model dimension independently and provide statistical robustness guarantees
under bounded Byzantine fractions~\cite{yin2018byzantine}. Distance-based
methods operate on whole update vectors. Krum selects the update closest to its
neighbors, Multi-Krum averages several such updates~\cite{blanchard2017machine},
and Bulyan combines Krum-style candidate selection with coordinate-wise trimmed
aggregation~\cite{el2018hidden}. Geometric-median methods provide another
robust alternative by selecting an aggregate that minimizes total distance to
client updates~\cite{pillutla2022robust}.

Other defenses use filtering, trust scores, or historical consistency. DnC
searches for suspicious directions in the submitted updates and removes
high-scoring clients~\cite{shejwalkar2021manipulating}. AFA assigns adaptive
weights based on similarity to the aggregate~\cite{munoz2019byzantine}.
FoolsGold targets sybil-style poisoning through update-similarity patterns
\cite{fung2020limitations}, FLTrust bootstraps trust from a small server-side
root dataset~\cite{cao2021fltrust}, and FLDetector detects clients whose
multi-round updates are inconsistent with historical behavior~\cite{zhang2022fldetector}.

These defenses rely on different assumptions about malicious behavior:
coordinate-wise methods look for abnormal values in individual dimensions,
distance-based methods rely on honest-client clustering, and filtering or
trust-based methods rely on detectable anomalous directions, similarity, or
temporal inconsistency. Modern poisoning attacks challenge these assumptions by
crafting updates that appear statistically or geometrically plausible while
still biasing the aggregate~\cite{baruch2019little,xie2020fall,fang2020local,
shejwalkar2021manipulating}.
\subsection{Error-Bounded Lossy Compression}
\label{subsec:compression-background}

Communication cost is a central bottleneck in federated learning because model
updates are high-dimensional and must be transmitted repeatedly across many
rounds. Compression is therefore commonly used to reduce the size of client
updates before transmission~\cite{wilkins2024fedsz}. In the setting considered
here, compression is part of the FL communication pipeline: clients compress
their updates, and the server decompresses them before aggregation.

We focus on error-bounded lossy compression (EBLC), where the compressor
controls reconstruction error according to a user-specified error bound. Given
an update $u$ and a compressor configuration with tolerance $\tau$ and
error-bound mode $m$, the client sends a compressed representation
$C_{\tau,m}(u)$. The server decompresses it as
$\bar{u} = D(C_{\tau,m}(u))$
where $D$ denotes decompression. The reconstruction difference between $u$ and
$\bar{u}$ is the compression-induced distortion.

We consider three representative EBLC compressors with different internal
mechanisms: \textbf{SZ2}, a prediction-based compressor~\cite{liang2018error}; \textsc{ZFP}, a transform-based compressor~\cite{lindstrom2014fixed,diffenderfer2019error}; and \textsc{TTHRESH}, a transform-based compressor~\cite{ballester2019tthresh}.





\subsection{Compression-aware robust aggregation}

Recent work has studied Byzantine robustness and communication compression
jointly. BROADCAST~\cite{Zhu_2023} shows that directly combining compressed
stochastic gradients with robust aggregation can suffer from compression noise
under Byzantine attacks, and reduces this effect through gradient-difference
compression and variance reduction. Rammal et al.~\cite{pmlr-v238-rammal24a}
develop Byzantine-robust compressed optimization algorithms with improved
convergence guarantees, including error-feedback and bi-directional compression
variants. Other methods use sparsification as part of the robust aggregation
pipeline. For example, LASA~\cite{Xu_2025_WACV} applies pre-aggregation
sparsification and then performs layer-adaptive filtering using magnitude and
direction information. Related compressed robust FL methods similarly treat
compression as a communication constraint or as a source of error that must be
controlled during robust learning~\cite{HU2026108582}.

The common theme is that compression is treated as a challenge to robustness:
its error must be reduced, compensated for, or designed around. To the best of
our knowledge, prior Byzantine-robust compressed learning methods do not use the
negative effects of compression, namely the structure of information loss under
compression, as a security signal against model-poisoning attacks.

\section{Threat Model}
\label{subsec:threat-model}

We assume an honest server and $n$ participating clients. Up to $f$ clients may
be malicious. The adversary controls these clients and can replace their honest
updates with arbitrary crafted updates. Malicious clients may coordinate with
one another and may know the current global model. This is the standard
Byzantine model-poisoning setting used to evaluate robust aggregation rules.

CRAFT assumes an honest-majority regime, namely $f < n/2$. This assumption is
necessary for any defense that relies on majority structure: if malicious
clients form a majority, they can define the dominant behavior observed by the
server. CRAFT does not assume that the server knows the exact number of
malicious clients.

The adversary's objective is untargeted performance degradation. After malicious
updates are included in the aggregation process, the resulting global model
should perform worse on the main test task. We do not consider backdoor attacks,
privacy attacks, client data compromise, server compromise, or attacks against
the compression library itself.

In the compressed federated learning pipeline considered here, clients submit
compressed update representations. The server decompresses all received updates
using the shared compressor configuration and then applies aggregation. CRAFT
uses the decompressed updates for model training and uses compression-footprint
statistics to estimate trust during aggregation.
\section{Compression-Footprint Separation of Malicious Updates}
\label{sec:compression-footprint-separation}

This section studies a different representation for robust aggregation based on
error-bounded lossy compression (EBLC). For each submitted update, we compare
the update to its decompressed output after a shared
compression--decompression pass and summarize that comparison as a
low-dimensional \emph{compression footprint}
We then show why this
representation can carry a useful separation signal between honest and
malicious updates.

\subsection{From Raw-Update Statistics to Compression Footprints}
\label{subsec:raw-to-footprint}

Let $g_i^{(t)} \in \mathbb{R}^d$ denote the update submitted by client $i$ at
communication round $t$. Our method changes the representation used by the
server for robust aggregation.

For a fixed EBLC compressor configuration $(\tau,m)$, where $\tau$ denotes the
compression tolerance and $m$ the error-bound mode, let $u_i^{(t)}$ denote the
input to the footprint computation and let
\[
\bar{u}_i^{(t)} = D(C_{\tau,m}(u_i^{(t)}))
\]
be its decompressed output under the shared compression--decompression
pipeline. We define the compression footprint as
\begin{equation}
x_i^{(t)} = \phi_{\tau,m}\!\left(u_i^{(t)}\right),
\label{eq:footprint-map}
\end{equation}
where $\phi_{\tau,m}(u)$ summarizes the relation between $u$ and its
decompressed output $\bar{u}=D(C_{\tau,m}(u))$.

More generally, for any input $u \in \mathbb{R}^d$, the compression footprint
is the five-dimensional vector
\begin{equation}
\phi_{\tau,m}(u)
=
\left[
r_{\tau,m}(u),
c_{\tau,m}(u),
q_{\tau,m}(u),
s_{\tau,m}(u),
\rho_{\tau,m}(u)
\right]^T,
\label{eq:footprint-vector}
\end{equation}
whose coordinates are defined in Table~\ref{tab:footprint-coordinates}.

\begin{table}[t]
\centering
\scriptsize
\renewcommand{\arraystretch}{1.12}
\setlength{\tabcolsep}{3pt}
\begin{tabularx}{\columnwidth}{
>{\raggedright\arraybackslash}p{0.17\columnwidth}
>{\raggedright\arraybackslash}p{0.24\columnwidth}
>{\raggedright\arraybackslash}X
}
\toprule
\textbf{Symbol} & \textbf{Name} & \textbf{Definition} \\
\midrule
$r_{\tau,m}(u)$
& Relative Distortion
& $\frac{\|u-\bar{u}\|_2}{\|u\|_2}$ \\
\midrule
$c_{\tau,m}(u)$
& Cosine Similarity
& $\frac{\langle u,\bar{u}\rangle}{\|u\|_2\|\bar{u}\|_2}$ \\
\midrule
$q_{\tau,m}(u)$
& Mean Squared Error
& $\frac{1}{d}\|u-\bar{u}\|_2^2$ \\
\midrule
$s_{\tau,m}(u)$
& Sparsity Change
& $\frac{1}{d}\sum_{j=1}^{d}\!\left(\mathbf{1}\{|\bar{u}_j|\leq\kappa\}-\mathbf{1}\{|u_j|\leq\kappa\}\right)$ \\
\midrule
$\rho_{\tau,m}(u)$
& Compression Ratio
& $\frac{B_{\mathrm{raw}}(u)}{B_{\tau,m}(u)}$ \\
\bottomrule
\end{tabularx}
\caption{The five coordinates of the EBLC compression footprint.}
\label{tab:footprint-coordinates}
\end{table}

Here, $d$ is the number of update coordinates, $\mathbf{1}\{\cdot\}$ is the
indicator function, and $\kappa$ is the near-zero threshold used to measure
sparsity change. The quantities $B_{\mathrm{raw}}(u)$ and $B_{\tau,m}(u)$ denote
the compressed and uncompressed payload sizes, respectively.

Relative Distortion, Cosine Similarity, and Mean Squared Error quantify
reconstruction fidelity between $u$ and $\bar{u}$. Sparsity Change measures how
compression changes the fraction of coordinates whose magnitude is below
$\kappa$, while Compression Ratio captures payload-size reduction. Compression
footprinting does not create information absent from the submitted update;
rather, it applies a nonlinear feature map that emphasizes
compressor-sensitive structure. If a poisoning attack alters this structure
while remaining inconspicuous under conventional raw-update statistics, then
the footprint representation can provide a more useful space for distinguishing
malicious clients from honest ones.

\subsection{Mathematical Motivation for Footprint Separability}
\label{subsec:footprint-separability-math}

We now formalize why EBLC can induce different footprint distributions for
honest and malicious clients.

For honest and malicious clients, we write the submitted updates as
\[
g_i^{H,t} = \mu^{(t)} + \xi_i^{(t)},
\qquad
g_i^{A,t} = g_i^{H,t} + \delta_i^{(t)},
\]
where $\mu^{(t)}$ is the mean honest update at round $t$,
$\xi_i^{(t)}$ captures natural honest variation due to stochastic
optimization and data heterogeneity, and $\delta_i^{(t)}$ is the perturbation
introduced by the attack.

The key point is that the EBLC error operator is input-dependent. For any
input $u$, the residual $E_{\tau,m}(u)=u-D(C_{\tau,m}(u))$ is not independent
random noise, nor is it determined solely by $\|u\|_2$. Because EBLC
reconstructs the input subject to a fixed error bound, both the residual and
the compressed payload size depend on how the values of $u$ interact with the
compressor. At the coordinate level we may write
$|E_{\tau,m,j}(u)| \leq b_{\tau,m,j}(u)$, where the effective error scale
$b_{\tau,m,j}(u)$ depends on both the compressor configuration and the input
itself.

This matters because two submitted updates may remain difficult to separate
under raw-space statistics yet respond differently to the same EBLC operator.
Honest updates are generated by local optimization of the same global
objective, so under IID or mildly heterogeneous data their value
distributions, reconstruction errors, angular changes, near-zero behavior, and
compressed sizes are expected to exhibit a dominant benign pattern. Malicious
updates, by contrast, are generated by an attack rule. The perturbation
$\delta_i^{(t)}$ can therefore alter the compressor response and induce a
shifted footprint distribution.

To reason at the distribution level, let
$G_H^{(t)} \sim \mathcal{P}_H^{(t)}$ and
$G_A^{(t)} \sim \mathcal{P}_A^{(t)}$ denote random honest and malicious
updates at round $t$. Their induced footprint centers are
\[
\mu_H^{\phi,t}
=
\mathbb{E}\!\left[\phi_{\tau,m}(G_H^{(t)})\right],
\qquad
\mu_A^{\phi,t}
=
\mathbb{E}\!\left[\phi_{\tau,m}(G_A^{(t)})\right],
\]
and the separation between them is
$
\Delta_\phi^{(t)}
=
\left\|
\mu_A^{\phi,t} - \mu_H^{\phi,t}
\right\|_2.
$

This formulation also clarifies the tradeoff between attack strength and
separability. If $\delta_i^{(t)}$ is too small, then the induced footprint
shift $\Delta_\phi^{(t)}$ may also be small, so malicious clients remain
difficult to distinguish from the honest population. However, such a weak
perturbation is also less likely to substantially affect the global model.
Conversely, attacks that move updates enough to significantly influence
aggregation are more likely to alter the EBLC response and therefore induce a
larger footprint shift. In this sense, footprint separability is 
tied to an impact--separation tradeoff.

We now state concentration conditions under which the honest and malicious
footprint distributions become separable. The radii introduced below are
theoretical concentration radii; they are not quantities directly computed by
the final algorithm.

\begin{assumption}[Dominant honest footprint concentration]
\label{assump:honest-footprint-concentration}
There exists a concentration radius $\mathcal{R}_H^{(t)}$ and a failure
probability $\gamma_H \in [0,1)$ such that
$\mathbb{P}(\|\phi_{\tau,m}(G_H^{(t)}) - \mu_H^{\phi,t}\|_2
\leq \mathcal{R}_H^{(t)}) \geq 1-\gamma_H$.
\end{assumption}

\begin{assumption}[Malicious footprint concentration]
\label{assump:attacker-footprint-concentration}
For a fixed attack and a fixed compressor configuration, there exists a
concentration radius $\mathcal{R}_A^{(t)}$ and a failure probability
$\gamma_A \in [0,1)$ such that
$\mathbb{P}(\|\phi_{\tau,m}(G_A^{(t)}) - \mu_A^{\phi,t}\|_2
\leq \mathcal{R}_A^{(t)}) \geq 1-\gamma_A$.
\end{assumption}

Intuitively, footprint separability requires the distance between the honest
and malicious footprint centers to exceed the combined spread of their
dominant regions.

\begin{proposition}[High-probability footprint separability]
\label{prop:footprint-separability}
Fix an EBLC compressor configuration $(\tau,m)$, and let
$\mu_H^{\phi,t}$ and $\mu_A^{\phi,t}$ denote the expected honest and
malicious footprint centers at round $t$. Suppose
Assumptions~\ref{assump:honest-footprint-concentration}
and~\ref{assump:attacker-footprint-concentration} hold. If
\[
\Delta_\phi^{(t)}
>
\mathcal{R}_H^{(t)} + \mathcal{R}_A^{(t)},
\]
then the dominant honest and malicious footprint regions are disjoint.
Consequently, a random honest footprint and a random malicious footprint are
separated in footprint space with probability at least
$1-\gamma_H-\gamma_A$.
\end{proposition}

\begin{proof}
By Assumptions~\ref{assump:honest-footprint-concentration}
and~\ref{assump:attacker-footprint-concentration}, and by the union bound,
with probability at least $1-\gamma_H-\gamma_A$ both concentration events hold
simultaneously. On that event, the triangle inequality gives
\[
\left\|
\phi_{\tau,m}(G_H^{(t)}) - \phi_{\tau,m}(G_A^{(t)})
\right\|_2
\geq
\Delta_\phi^{(t)} - \mathcal{R}_H^{(t)} - \mathcal{R}_A^{(t)}.
\]
If $\Delta_\phi^{(t)} > \mathcal{R}_H^{(t)} + \mathcal{R}_A^{(t)}$, then the
right-hand side is strictly positive, so the dominant honest and malicious
footprint regions do not overlap.
\end{proof}

Proposition~\ref{prop:footprint-separability} is a sufficient condition. It identifies the regime in which EBLC provides a useful
signal: the attack-induced footprint shift must exceed the combined spread of
the dominant honest and malicious footprint regions. This regime is most
plausible when honest clients remain sufficiently coherent in footprint space
and the attack perturbs compressor-sensitive structure in a systematic way.

\subsection{EBLC vs. Top-$K$: Structural Difference}
\label{subsec:why-not-topk}

Top-$K$~\cite{shi2019understanding} sparsification is also lossy, but the footprint it induces is much
narrower than the footprint induced by EBLC. Let $S_K(g)$ denote the set of
indices corresponding to the $K$ largest-magnitude coordinates of
$g \in \mathbb{R}^d$. Under Top-$K$, the reconstruction keeps $g_j$ unchanged
for $j \in S_K(g)$ and sets all other coordinates to zero. Therefore,
\[
r_{\mathrm{TopK}}(g)^2
=
\frac{\sum_{j \notin S_K(g)} g_j^2}{\sum_{j=1}^{d} g_j^2},
\qquad
c_{\mathrm{TopK}}(g)^2
=
\frac{\sum_{j \in S_K(g)} g_j^2}{\sum_{j=1}^{d} g_j^2},
\]
and hence $c_{\mathrm{TopK}}(g)^2 + r_{\mathrm{TopK}}(g)^2 = 1$. Thus,
Relative Distortion and Cosine Similarity are not independent under Top-$K$;
they are two views of the same retained-versus-discarded squared magnitude
ratio. Mean Squared Error is governed by the same mechanism, since
$q_{\mathrm{TopK}}(g)=\frac{\|g\|_2^2}{d}r_{\mathrm{TopK}}(g)^2$. The
remaining coordinates are also weak under fixed-$K$ sparsification: Sparsity
Change is driven mainly by the fact that exactly $K$ entries are retained,
while Compression Ratio is determined largely by storing those $K$ values
together with their indices.

Therefore, the Top-$K$ footprint is dominated by one narrow source of
variation, namely concentration between retained and discarded coordinates.
EBLC behaves differently. Its residual and payload size are not determined
solely by discarded energy; they also depend on how the update interacts with
the compressor's prediction, quantization, and error-bound mechanisms.
Consequently, the induced footprint depends on richer structural properties
such as local variation, dynamic range, predictability, and the distribution
of reconstruction error across coordinates.

\section{Methodology}
\label{sec:methodology}

We operationalize the our central hypothesis in CRAFT (\emph{Compression-guided Robust Aggregation via Footprint Trust}), a server-side robust aggregation method that uses compression footprints to
assign continuous trust weights before aggregation. Its motivation is both
systems-oriented and security-oriented. In practical FL,
client-to-server communication is often a dominant bottleneck, so update
compression is already required as a systems primitive. CRAFT leverages this
same compression stage not only for communication efficiency, but also as a
source of structural signal for robustness.

CRAFT utilizes the SZ2 EBLC in relative-error mode. 
CRAFT builds on three design choices. First, the
server never relies on client-reported compression statistics, since Byzantine
clients can submit poisoned updates together with misleading metadata. Second,
all footprint measurements are recomputed by a server-controlled
compression--decompression pass applied uniformly to every received update.
Third, CRAFT does not require the number of attackers to be known in advance.
Instead, it assigns each client a continuous trust weight according to
consistency with the dominant footprint pattern in the current round.

\begin{figure*}[t]
\centering
\includegraphics[width=.9\textwidth]{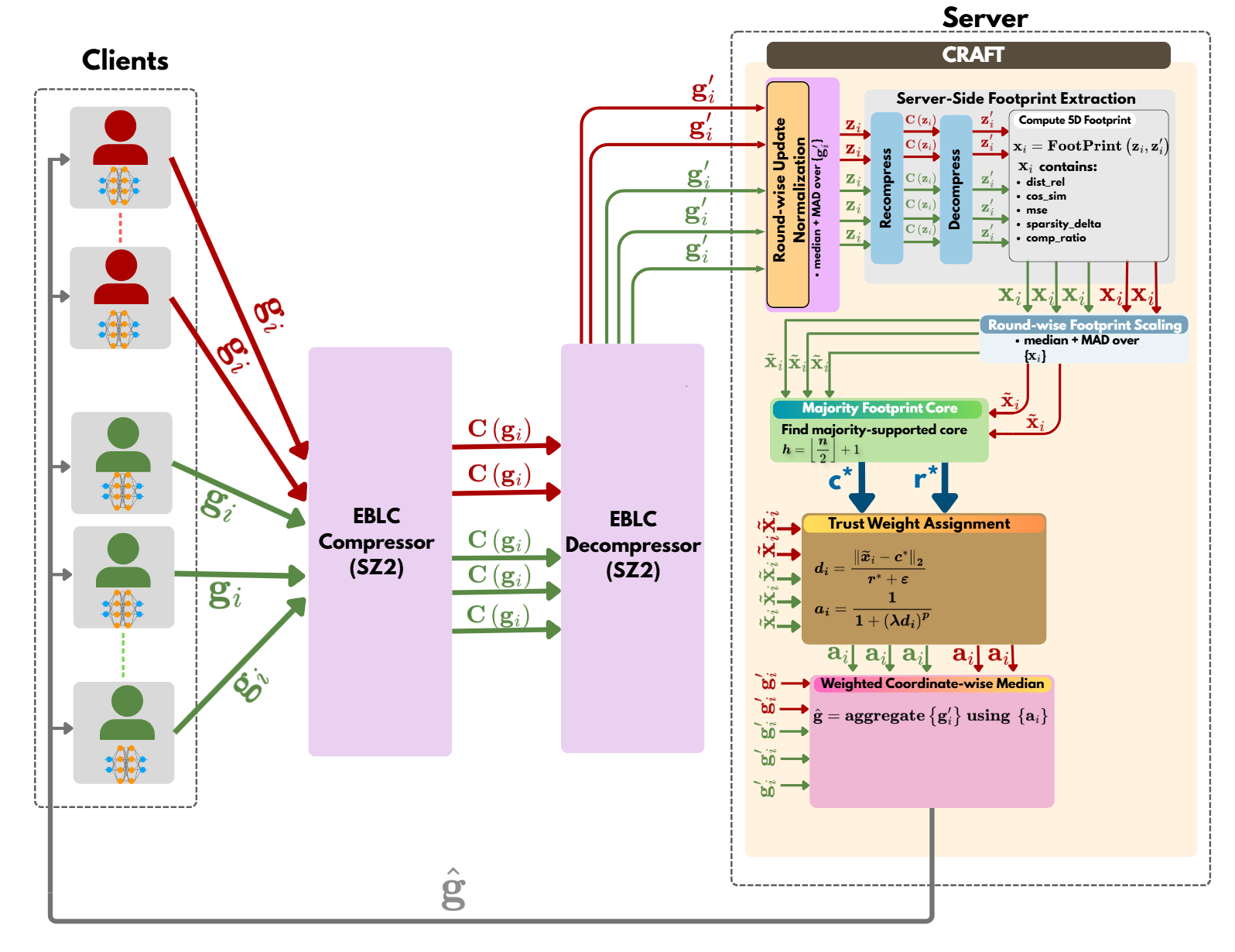}
\caption{Overview of CRAFT. The server decompresses client updates, extracts
server-side compression footprints, assigns footprint-based trust weights, and
aggregates the decompressed updates with a weighted coordinate-wise median.}
\label{fig:craft-pipeline}
\end{figure*}

Figure~\ref{fig:craft-pipeline} summarizes the full pipeline. At a fixed
communication round, let $g_i$ denote the client-side pre-compression update,
let
$y_i = C_{\tau,m}(g_i)$
denote the transmitted compressed payload, and let
$g_i' = D(y_i)$
denote the server-side decompressed update. These decompressed updates
$\{g_i'\}_{i=1}^n$ are the only updates used in the final aggregation step.

CRAFT then constructs a diagnostic footprint path from these decompressed
updates. It forms robustly normalized updates $\{z_i\}_{i=1}^n$, computes
server-side decompressed recompression outputs
$\{z_i'\}_{i=1}^n$ with $z_i' = D(C_{\tau,m}(z_i))$, extracts five-dimensional
footprints
\[
x_i = \phi_{\tau,m}(z_i) = \mathrm{Footprint}(z_i,z_i'),
\]
and robustly scales them to obtain $\{\tilde{x}_i\}_{i=1}^n$. It then
identifies a majority-supported footprint core, assigns trust weights, and
aggregates the decompressed updates $\{g_i'\}_{i=1}^n$ using a weighted
coordinate-wise median. To keep the notation readable, we suppress the round
index in this section whenever the round is fixed and clear from context.

Algorithm~\ref{alg:craft} gives a high-level summary; the remaining
subsections describe each stage in detail.

\begin{algorithm}[t]
\footnotesize
\caption{CRAFT: Compression-guided Robust Aggregation via Footprint Trust}
\label{alg:craft}
\begin{algorithmic}[1]
\REQUIRE Compressed payloads $\{y_i\}_{i=1}^{n}$, optional base aggregation weights $\{n_i\}_{i=1}^{n}$ (default $n_i=1$), compressor $C_{\tau,m}$, decompressor $D$, trust parameters $\lambda,p$
\ENSURE Aggregated update $\hat{g}$

\STATE Decompress each received payload: $g_i' \leftarrow D(y_i)$ for
$i=1,\dots,n$

\STATE Compute coordinate-wise robust centers and scales from
$\{g_i'\}_{i=1}^{n}$ and normalize the updates to obtain $\{z_i\}_{i=1}^{n}$

\FOR{each client $i$}
\STATE Compute $z_i' \leftarrow D(C_{\tau,m}(z_i))$
\STATE Compute footprint $x_i \leftarrow \phi_{\tau,m}(z_i)$
\ENDFOR

\STATE Robustly scale the footprint matrix to obtain
$\{\tilde{x}_i\}_{i=1}^{n}$

\STATE Set $h \leftarrow \lfloor n/2 \rfloor + 1$

\STATE For each client $i$, compute the radius $r_i$ needed to include its
$h$ nearest clients in the scaled footprint space

\STATE Select the majority-supported core:
$i^\star \leftarrow \arg\min_i r_i$,
$c^\star \leftarrow \tilde{x}_{i^\star}$,
$r^\star \leftarrow r_{i^\star}$

\FOR{each client $i$}
\STATE Compute normalized distance:
$d_i \leftarrow \|\tilde{x}_i - c^\star\|_2/(r^\star + \epsilon)$
\STATE Assign trust weight:
$a_i \leftarrow 1/(1+(\lambda d_i)^p)$
\STATE Set aggregation weight:
$\alpha_i \leftarrow a_i n_i$
\ENDFOR

\STATE Aggregate decompressed updates: \\
\quad \ \ $\hat{g} \leftarrow \mathrm{WeightedCoordMedian}(\{g_i'\}_{i=1}^{n}, \{\alpha_i\}_{i=1}^{n})$

\STATE \textbf{return} $\hat{g}$

\end{algorithmic}
\end{algorithm}

\subsection{Robust Round-Wise Update Normalization}
\label{subsec:update-normalization}

Before extracting footprints, CRAFT applies a robust round-wise normalization
to the decompressed updates $\{g_i'\}_{i=1}^n$. This step is needed because
server-side footprint extraction is performed on updates that have already
been decompressed once. Without normalization, the second server-controlled
compression pass can be dominated by raw coordinate scale rather than by the
structural behavior that CRAFT is intended to measure.

Let $m \in \mathbb{R}^d$ and $s \in \mathbb{R}^d$ denote the coordinate-wise
median and median absolute deviation (MAD) of $\{g_i'\}_{i=1}^n$. CRAFT
defines
\begin{equation}
z_i
=
\frac{g_i' - m}{s + \epsilon},
\label{eq:update-normalization}
\end{equation}
where division is understood coordinate-wise and $\epsilon>0$ is a small
stability constant. This transformation places all decompressed updates in a
common robust coordinate system and encourages the subsequent recompression
step to respond to relative structural behavior rather than raw magnitude.

\subsection{Server-Side Compression-Footprint Extraction}
\label{subsec:server-footprint}

After normalization, the server performs a second controlled
compression--decompression pass on each update:
\begin{equation}
z_i' = D(C_{\tau,m}(z_i)).
\label{eq:server-recompression}
\end{equation}
The server then extracts the corresponding footprint
\begin{equation}
x_i = \phi_{\tau,m}(z_i) = \mathrm{Footprint}(z_i,z_i') \in \mathbb{R}^5,
\label{eq:server-footprint}
\end{equation}
where $\phi_{\tau,m}$ is the compression-footprint map defined in
Section~\ref{sec:compression-footprint-separation}. The components are
the same five components introduced there: relative distortion,
cosine similarity, mean squared error, sparsity change, and compression ratio.

These five components summarize complementary aspects of how an update behaves
under the fixed EBLC operator. Relative distortion and mean squared error
measure decompression loss; cosine similarity measures directional
preservation; sparsity change captures how recompression alters near-zero
structure; and compression ratio reflects how efficiently the update can be
encoded under the chosen error bound. Under the same server-controlled EBLC
pass, benign and attack-modified updates can induce different residual
structure, near-zero behavior, and compressibility, and these differences are
exactly what the footprint captures.

\subsection{Robust Round-Wise Footprint Scaling}
\label{subsec:footprint-scaling}

The five footprint components naturally have different scales. If distances
were computed directly in raw footprint space, one component could dominate
simply because of its numerical range. CRAFT therefore applies a second robust
median/MAD normalization across clients in the footprint space.

Let
$X = [x_1,\dots,x_n]^T \in \mathbb{R}^{n\times 5}$
denote the footprint matrix. For each footprint component
$j \in \{1,\dots,5\}$, CRAFT computes
\begin{equation}
\tilde{x}_{ij}
=
\frac{
x_{ij} - \operatorname{median}_{k}(x_{kj})
}{
\operatorname{MAD}(\{x_{kj}\}_{k=1}^n) + \epsilon
}.
\label{eq:footprint-scaling}
\end{equation}
This produces the scaled footprint matrix
$\tilde{X}=[\tilde{x}_1,\dots,\tilde{x}_n]^T$, in which all five components
contribute on a comparable robust scale.

\subsection{Majority-Supported Footprint Trust Assignment}
\label{subsec:trust-assignment}

At a fixed round, CRAFT assigns client influence in the scaled footprint space
$\{\tilde{x}_i\}_{i=1}^n$. This step operationalizes the separation regime
characterized in
Assumptions~\ref{assump:honest-footprint-concentration}
and~\ref{assump:attacker-footprint-concentration} and in
Proposition~\ref{prop:footprint-separability}. When the honest and malicious
footprint distributions are sufficiently separated, the honest clients induce
the dominant compact region in footprint space. CRAFT estimates this region
empirically through a majority-supported core.

CRAFT assumes a strict honest majority, that is, fewer than half of the
participating clients are Byzantine. This is the same breakdown regime assumed
by standard majority-based robust aggregation rules. Let
$h = \left\lfloor \frac{n}{2} \right\rfloor + 1$.
For each client $i$, define
\begin{equation}
r_i
=
d_h\!\left(\tilde{x}_i,\{\tilde{x}_j\}_{j=1}^n\right),
\label{eq:majority-radius}
\end{equation}
where $d_h(\cdot)$ denotes the distance to the $h$th nearest neighbor in the
scaled footprint space. Equivalently, $r_i$ is the smallest radius of a ball
centered at $\tilde{x}_i$ that contains a strict majority of clients. CRAFT
then selects
\[
i^\star = \arg\min_i r_i, \qquad
c^\star = \tilde{x}_{i^\star}, \qquad
r^\star = r_{i^\star}.
\]

The point $c^\star$ is the empirical majority-supported core and $r^\star$ is
the corresponding majority radius. Because $h$ is a strict majority, an
attacker-only cluster cannot define $(c^\star,r^\star)$ unless Byzantine
clients control more than half of the participating clients. Under the
honest-majority assumption, the majority-supported core must therefore be
anchored by the dominant region containing honest participation.

A malicious client can nevertheless be selected as the center index $i^\star$
if its footprint lies inside the same dominant region as the honest majority
and yields the smallest majority-supported radius. This does not mean that
CRAFT has identified a malicious cluster as benign. Rather, it means that, in
that round, the selected point is geometrically representative of the same
dominant footprint region as nearby honest clients. Since CRAFT uses $c^\star$
as a geometric anchor rather than as an estimate of client identity, centering
trust around such a point remains consistent with the majority geometry.

Once the core is determined, CRAFT assigns trust according to normalized
distance from that core:
\begin{equation}
d_i
=
\frac{\|\tilde{x}_i - c^\star\|_2}{r^\star + \epsilon},
\qquad
a_i
=
\frac{1}{1 + (\lambda d_i)^p},
\label{eq:trust-weight}
\end{equation}
where $\lambda>0$ controls how quickly trust decreases as a client moves away
from the majority-supported footprint core, and $p\geq1$ controls the curvature
of this decay. In practice, \(\lambda\) should be calibrated relative to the
normalized majority radius: smaller values give a more tolerant weighting rule,
while larger values downweight clients soon after they leave the core. The
parameter \(p\) controls how gradual or abrupt this transition is; moderate
values avoid making trust overly sensitive to small footprint fluctuations near
the boundary.

This soft weighting is important when the footprint separation is imperfect.
If Proposition~\ref{prop:footprint-separability} is strongly satisfied, then
malicious footprints lie outside the dominant majority-supported region and
receive very small weights. If the honest population is more dispersed, as can
happen under stronger non-IID heterogeneity, then the empirical majority
radius becomes broader and some malicious clients may lie inside or near the
boundary of $r^\star$. Likewise, if the attack-induced footprint shift is
small relative to the combined honest and malicious spread, then malicious
footprints may overlap more with the honest population. In these cases CRAFT
does not force a brittle binary decision. Instead, it reduces influence in
proportion to inconsistency with the dominant footprint pattern.

This interpretation is also consistent with the impact--separation tradeoff
discussed in Section~\ref{subsec:footprint-separability-math}. Attacks that
induce a larger systematic footprint shift are easier to downweight. Attacks
that remain embedded in the dominant footprint region are harder to
distinguish, but they also induce less geometric separation under the common
compression operator. Accordingly, CRAFT is best aligned with the IID or
mildly heterogeneous regime, where the dominant honest footprint core is most
stable.

Figure~\ref{fig:trust-weight-examples} illustrates the trust-weighting
mechanism under the IPM attack on CIFAR-10.

\begin{figure}[t]
\centering
\begin{subfigure}[t]{0.97\linewidth}
\centering
\includegraphics[width=\linewidth]{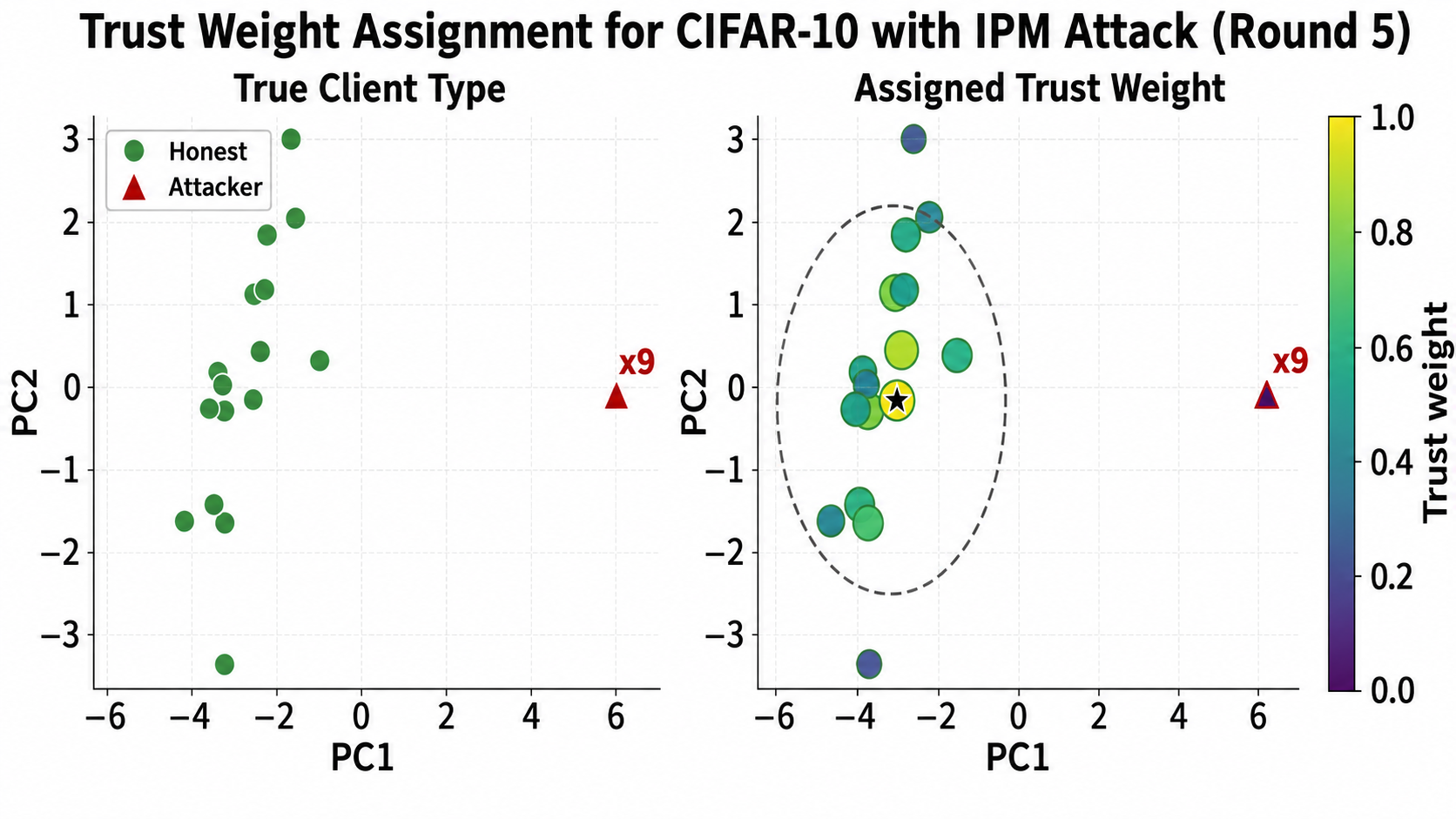}
\caption{Clear footprint separation.}
\label{fig:trust-clear-case}
\end{subfigure}

\vspace{0.25em}

\begin{subfigure}[t]{0.97\linewidth}
\centering
\includegraphics[width=\linewidth]{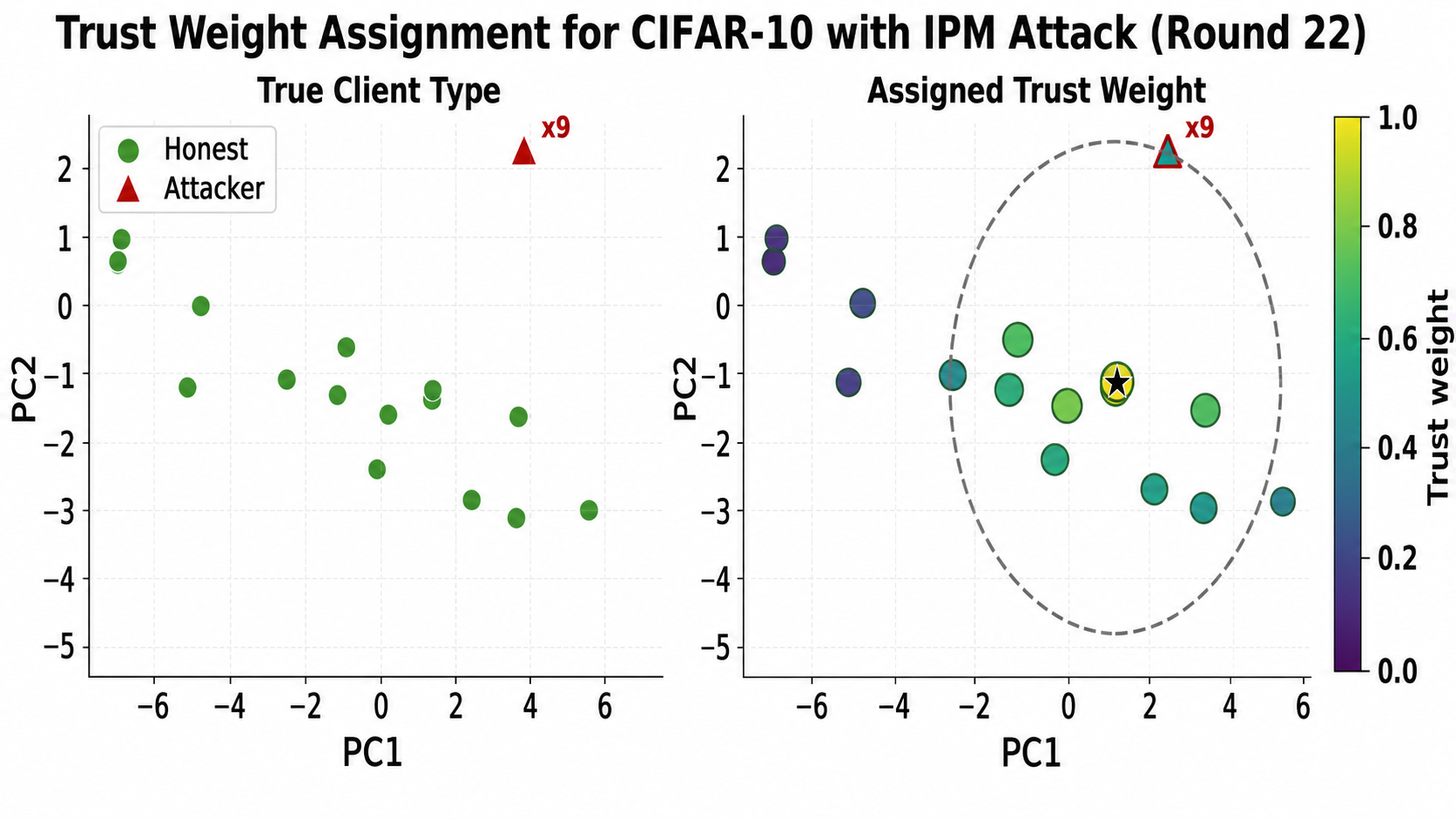}
\caption{Boundary case with reduced attacker weight.}
\label{fig:trust-boundary-case}
\end{subfigure}

\caption{Examples of CRAFT trust-weight assignment under the IPM attack on CIFAR-10.}
\label{fig:trust-weight-examples}
\end{figure}

\subsection{Trust-Weighted Robust Aggregation}
\label{subsec:weighted-median}

The footprint path is diagnostic only. CRAFT does not aggregate the normalized
updates $\{z_i\}_{i=1}^n$ or the server-side decompressed recompression outputs
$\{z_i'\}_{i=1}^n$. The final aggregation is performed on the decompressed
client updates $\{g_i'\}_{i=1}^n$, since these are the updates actually
recovered from the received client payloads.

CRAFT can be viewed as a trust-weighting layer that can be paired with a robust
aggregation rule capable of using client weights. Let $n_i$ denote the base
aggregation weight of client $i$, such as its number of local examples; for
uniform weighting, one can set $n_i=1$ for all clients. CRAFT combines this
base weight with the footprint trust score as
$\alpha_i = a_i n_i$ .
In this work, we instantiate the final aggregation rule as a weighted
coordinate-wise median over $\{g_i'\}_{i=1}^n$ using weights
$\{\alpha_i\}_{i=1}^n$. For each model coordinate, the weighted median selects
the value at which the cumulative weight first reaches half of the total
weight.

This final step provides robustness at two levels. The footprint weights reduce
the influence of updates whose compression behavior is inconsistent with the
dominant footprint pattern, while the weighted median limits the effect of
remaining high-magnitude or strategically placed values in individual
coordinates.
\section{Experiment Setup}
\label{sec:evaluation}

We now evaluate CRAFT and how well the structure of compression-induced loss can be used
as a security signal to distinguish honest and malicious updates. Unless otherwise stated,
all experiments use IID client partitions, which is the primary regime targeted by CRAFT:
under IID data, honest updates are more likely to induce a compact majority-supported
footprint region, making distance from that region a useful trust signal. Unless otherwise
stated, we simulate 25 clients with 9 malicious clients (36\%), preserving an honest majority. For CRAFT, we use the same trust-weight parameters in all main experiments:
\(\lambda=2\) for the footprint-distance scale and \(p=4\) for the trust-decay
power. Appendix~\ref{appendix:trust-parameter-ablation} reports an ablation over
these parameters on CIFAR-10 under IPM attack.

Our evaluation process is as follows.
First, we examine whether lossy compressors induce attack-sensitive footprints and how compressor
choice and tolerance affect the signal. Second, we evaluate whether footprint-derived trust
suppresses malicious influence and enables competitive end-to-end robustness. Finally, we test
CRAFT under non-IID data and a defense-aware attack, and report its server-side cost.

\label{subsec:experimental-setup}

\subsection{Datasets}
We evaluate on CIFAR-10, Fashion-MNIST, and Purchase. CIFAR-10 contains
60,000 color images from 10 object classes, with 50,000 training images and
10,000 test images. Fashion-MNIST contains 70,000 grayscale images from 10
clothing categories, with 60,000 training images and 10,000 test images.
Purchase is a tabular classification dataset derived from customer shopping
records, where each example is represented by a binary feature vector
indicating purchased items and the task is to predict the customer's purchase
class. Together, these datasets allow us to evaluate CRAFT across both image and
tabular learning tasks, rather than restricting the evaluation to a single data
type.

For CIFAR-10, we use a seven-layer CNN with two convolutional layers
(16 and 64 channels), max-pooling after each convolution, two fully connected
hidden layers of sizes 384 and 192, and a final 10-class output layer. For
Fashion-MNIST, we use a compact CNN with four $3\times3$ convolutional layers
(two with 32 channels and two with 64 channels), group normalization, max
pooling, dropout, and a final fully connected classifier. For Purchase, we use
an MLP with hidden dimensions 512, 256, and 128, ReLU activations, and a final
classification layer over the classes.

\subsection{Compressors}
Unless otherwise stated, CRAFT utilizes the SZ2 compressor as the default compressor.
We additionally evaluate CRAFT with
two other state-of-the-art EBLCs, ZFP and TThresh, as well as with Top-$K$ sparsification.
CRAFT uses SZ2 in relative error-bound mode with tolerance $\tau=0.01$. We choose
this setting because FedSZ~\cite{wilkins2024fedsz} reports that $\tau=0.01$ in
SZ2 relative mode provides the best accuracy--compression tradeoff for
federated learning.  Relative error bounds scale compression tolerance to each coordinate's local magnitude rather than applying a fixed tolerance uniformly. This suits CRAFT because its footprint metric depends on how updates change under compression given their distributional structure.

\subsection{Poisoning attacks}

We evaluate six attacks: ALIE, IPM, Min-Max, Min-Sum, Sine, and PoisonedFL.
Together, these attacks stress different assumptions made by robust aggregation:
ALIE perturbs updates using honest-gradient statistics, IPM reverses the honest
update direction, Min-Max and Min-Sum are constructed to evade distance-based
filters, and Sine and PoisonedFL introduce structured model-poisoning behavior.
For ALIE, we use \(z=1\) in all experiments. For IPM, we use
\(\epsilon=10\), so malicious clients submit a scaled update in the opposite
direction of the mean honest update. These values follow the commonly used
strong-attack settings in the corresponding attack evaluations, where they are
large enough to significantly degrade standard aggregation while still producing
updates that are useful for stress-testing robust defenses. The remaining attacks
use their standard construction without an additional manually swept strength
parameter.

\subsection{Defense Baselines}
All defenses are evaluated under the
same client population and attack configuration. We compare CRAFT with six aggregation baselines: Mean, Median, Trimmed Mean,
Krum, Bulyan, and DnC. All methods are evaluated under the same
compressed-update pipeline. Clients transmit SZ2-compressed FedSGD updates, and
the server decompresses the received updates before applying the selected
aggregation rule. CRAFT uses the same decompressed updates for aggregation, but
also performs server-side footprint extraction to assign trust weights. 

\section{Results}
\label{sec:evaluation}

\subsection{Footprint Separation Analysis}
\label{subsec:mechanism-analysis}

We first evaluate our central hypothesis: does the structure of the information
loss induced by lossy compressors provide clear footprint separation to distinguish
honest and malicious clients? To answer this question, the following experiments
evaluate the choice of compressor, the tolerance of the EBLC pass, and the majority-core 
trust assignment.


\subsubsection{EBLC versus Top-\(K\) Footprint Structure}

We compare the footprint structure induced by EBLC/SZ2 and Top-\(K\)
sparsification on CIFAR-10 under the ALIE attack with \(9\) malicious clients
out of \(25\). For SZ2, we use relative error tolerance \(\tau=0.01\). For Top-\(K\), we retain the largest \(1\%\) of coordinates by absolute magnitude. For each round, we compute the five footprint
coordinates for every submitted client update and apply the same round-wise
robust normalization used by CRAFT before measuring honest--malicious
separation. The statistics in Figure~\ref{fig:effectsize-topk-eblc} are pooled
across the evaluated rounds; the corresponding per-metric distributions are
shown in Appendix~\ref{fig:violin-topk-eblc}.

Figure~\ref{fig:effectsize-topk-eblc} shows that both compression mechanisms can
produce some honest--malicious separation, but the structure of the signal is
different. Top-\(K\) separation is concentrated in a narrower set of metrics:
relative distortion and MSE show positive gaps of about \(3.32\) and \(3.73\),
while compression ratio is essentially uninformative. EBLC/SZ2 produces a
broader footprint response. Its strongest signal appears in sparsity change
(\(21.54\)), but cosine similarity (\(7.17\)), relative distortion
(\(-9.64\)), MSE (\(-4.38\)), and compression ratio (\(-4.71\)) also shift
substantially. The sign only indicates which group has the larger normalized
mean; the magnitude indicates separation strength.

These results support the use of EBLC footprints rather than a single
sparsification statistic. Top-\(K\) acts mainly through retained-versus-discarded
magnitude, whereas SZ2 changes the update through prediction, quantization, and
entropy coding. As a result, SZ2 exposes multiple complementary views of the
update's numerical structure, giving the footprint representation more ways to
separate malicious updates from the honest majority.
    
\subsubsection{Choice of EBLC}

Figure~\ref{fig:compressor-footprint-amplification} compares SZ2, ZFP, and
TTHRESH as footprint extractors. The x-axis shows the attack type. The y-axis
reports a footprint-amplification ratio: the separation obtained from the final
five-dimensional footprint divided by the separation obtained from the
compressor residual alone. Therefore, larger values indicate that the full
footprint features make the honest--malicious difference more visible than the
raw residual signal by itself.

SZ2 gives the largest amplification for all tested attacks. This suggests that
SZ2 is a better probe for the update structure relevant to our defense. SZ2
uses prediction and error-bounded quantization, so its output is sensitive not
only to the magnitude of the update, but also to how predictable the update is,
how its values vary locally, how much near-zero structure changes, and how well
the update can be compressed. Malicious updates can alter these properties even
when they are not extreme in the original gradient space. As a result, the SZ2
footprint produces stronger honest--malicious separation than the ZFP and
TTHRESH footprints in our experiments.

Min-Max is the weakest case in the comparison. This is expected because
Min-Max is designed to remain close to honest updates under distance-based
criteria, so the compressor-induced differences are smaller. Even in this
harder setting, SZ2 still produces the strongest footprint amplification among
the tested compressors.

\begin{figure}[t]
\centering
\includegraphics[width=\columnwidth]{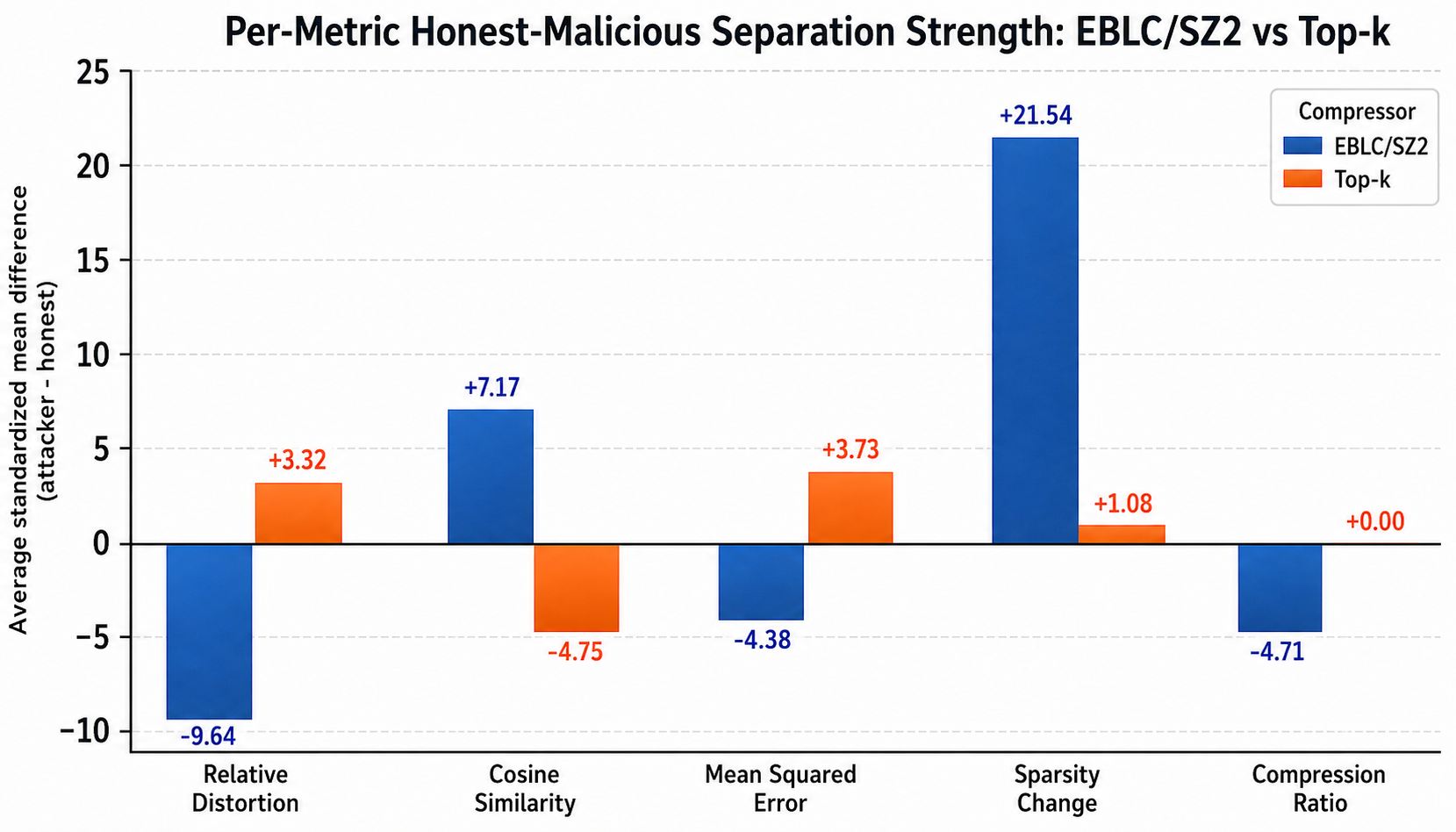}
\caption{Per-metric honest--malicious footprint separation for EBLC/SZ2 and
Top-$K$.}
\label{fig:effectsize-topk-eblc}
\end{figure}

\begin{figure}[t]
\centering
\includegraphics[width=\columnwidth]{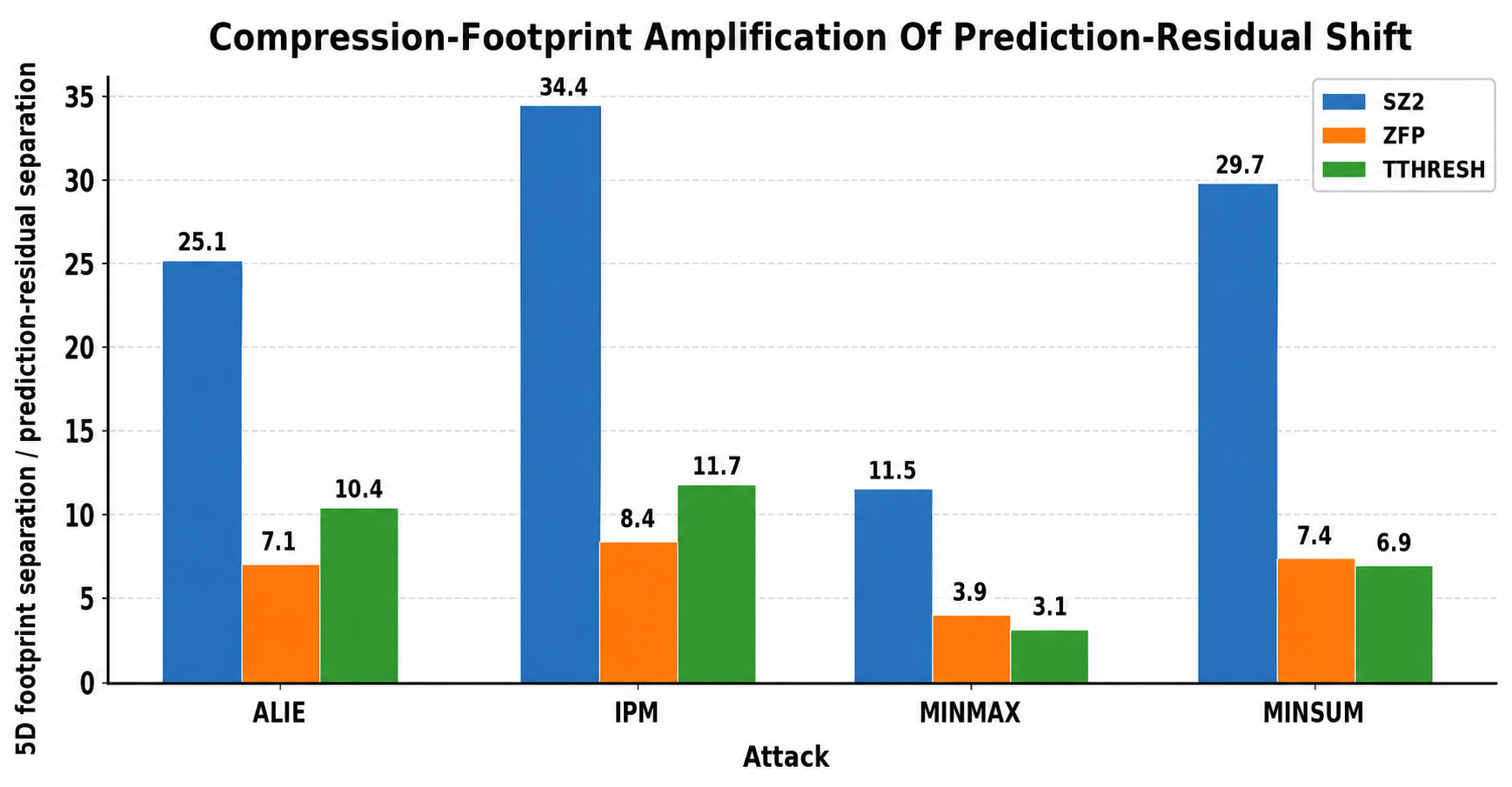}
\caption{EBLC-dependent footprint amplification across attacks; SZ2 gives
the strongest amplification in this setting.}
\label{fig:compressor-footprint-amplification}
\end{figure}

\begin{figure}[t]
\centering
\includegraphics[width=\columnwidth]{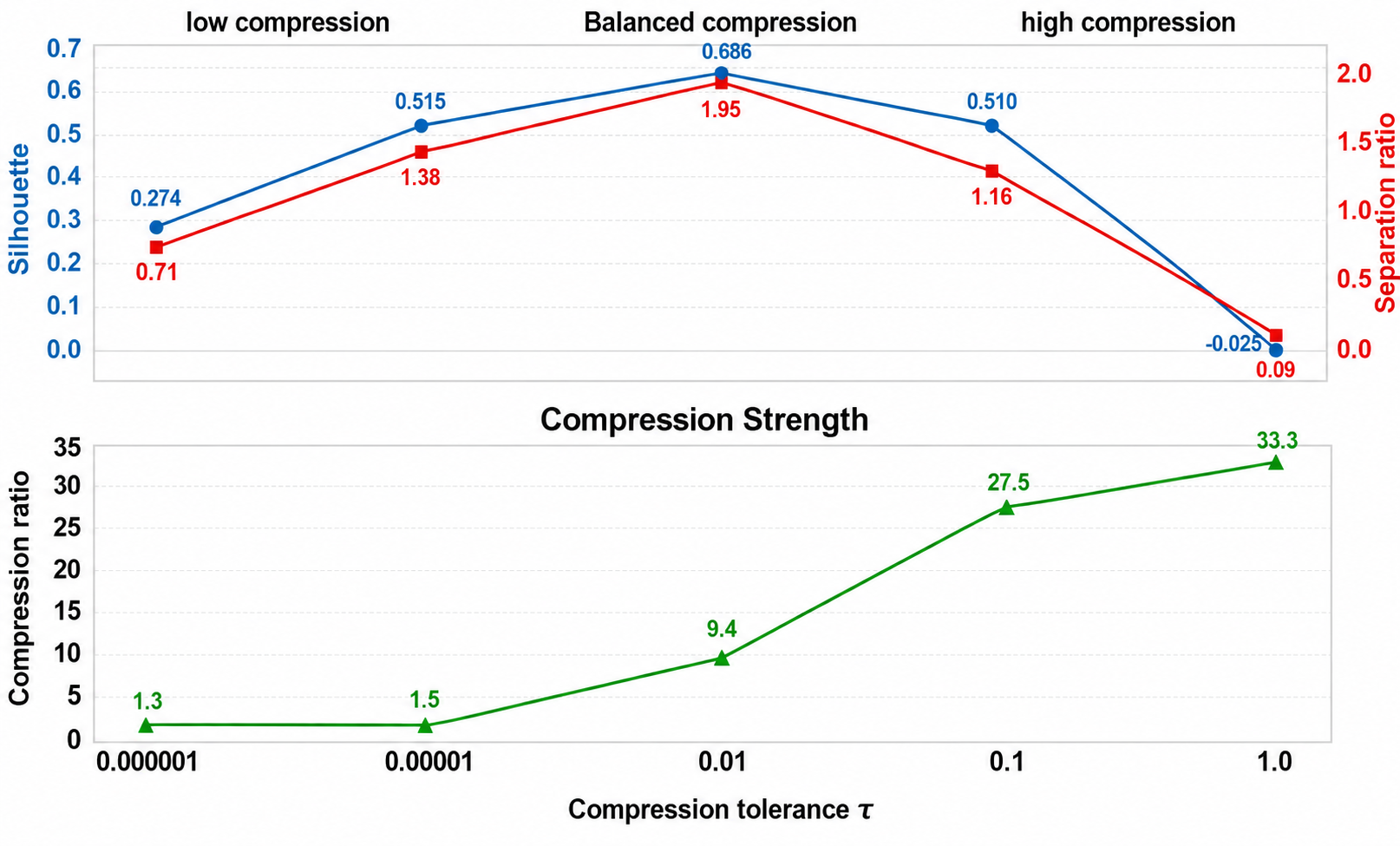}
\caption{Effect of SZ2 tolerance on footprint separation and compression
strength.}
\label{fig:tolerance-footprint-separation}
\end{figure}

\subsubsection{Tolerance and Footprint Signal}
Figure~\ref{fig:tolerance-footprint-separation} studies how the SZ2 tolerance
affects footprint separation. The result shows a non-monotonic relationship.
When the tolerance is too small, compression is weak and the decompressed update
is nearly identical to the input. In that regime, the footprint has little
signal because both honest and malicious updates appear almost unchanged by the
compression pass.

When the tolerance is too large, compression becomes too aggressive. The
compressed representation removes useful structure from both honest and
malicious updates, and the footprint signal again weakens. The useful region is
therefore in the middle: the compressor must perturb the update enough to
expose structural differences, but not so much that it destroys the geometry
needed for reliable comparison.

This observation is consistent with the communication--accuracy tradeoff
reported by FedSZ~\cite{wilkins2024fedsz}, which identifies SZ2 relative mode
with tolerance $\tau=0.01$ as a strong operating point for federated learning.
Our results show that the same tolerance is also meaningful from the defense
perspective: it provides useful compression while preserving enough structure
for the footprint to separate honest and malicious updates.

The lower part of the figure shows the corresponding compression ratio. As
expected, larger tolerances increase compression strength. The key point is
that the strongest observed footprint-signal operating point is not simply the largest compression ratio.
CRAFT needs a tolerance that balances communication reduction with diagnostic
footprint quality. The region around $\tau=0.01$ provides this balance in our
experiments.

\subsubsection{Footprint signal as a function of attack}

Figure~\ref{fig:craft-metric-contribution} shows that the footprint signal is
not carried equally by all five metrics. The dominant coordinate depends on the
attack. On CIFAR-10, Sparsity Change is the largest contributor for ALIE,
Min-Max, and Min-Sum, accounting for about $68.4\%$, $47.9\%$, and $63.4\%$ of
the separation, respectively. In contrast, IPM is dominated by Mean Squared
Error, which contributes about $54.2\%$, while Sparsity Change contributes only
about $7.1\%$. Sine and PoisonedFL are also strongly MSE-driven, with MSE
contributing about $86.7\%$ and $72.8\%$, respectively.

This attack-dependent behavior is the main point of the figure. ALIE and the
distance-constrained attacks often change how concentrated the update remains
after compression, making Sparsity Change informative. IPM and the more
structured attacks instead alter reconstruction error more strongly, making MSE
the dominant signal. Aggregated over all dataset--attack settings in our
analysis, MSE is the largest average contributor ($46.1\%$), followed by
Sparsity Change ($24.1\%$) and Compression Ratio ($13.5\%$). These results
support using a multi-metric footprint: no single compression statistic explains
all attacks.

\begin{figure}[t]
\centering
\includegraphics[width=\columnwidth]{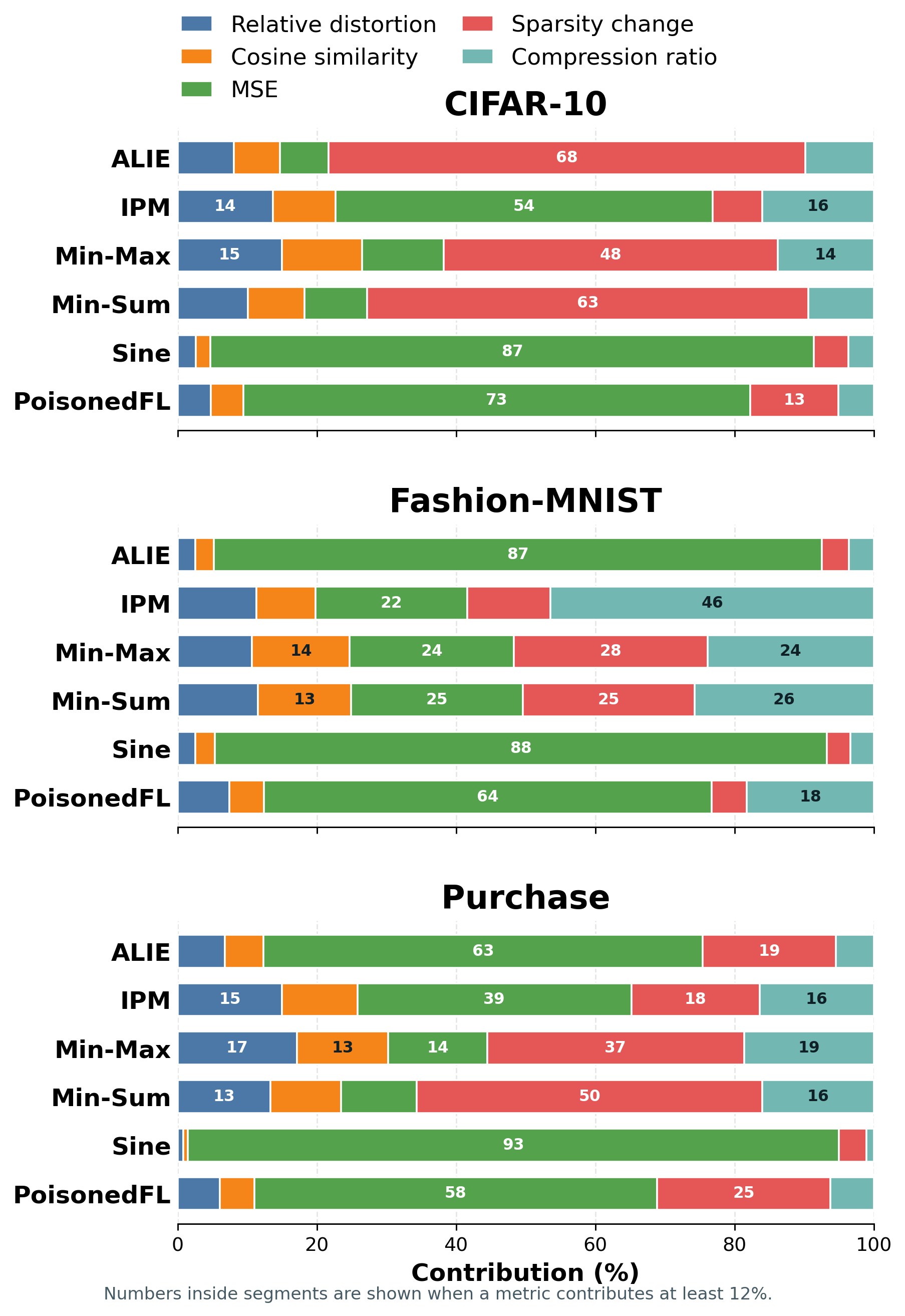}
\caption{Contribution of each compression-footprint metric to honest--malicious separation across attacks.}
\label{fig:craft-metric-contribution}
\end{figure}

\subsection{Accuracy Under Model-Poisoning Attacks}
\label{subsec:accuracy-under-attacks}

We next evaluate whether CRAFT preserves model utility under different
model-poisoning attacks. Figures~\ref{fig:cifar10-model-poisoning-accuracy},
\ref{fig:fmnist-model-poisoning-accuracy}, and
\ref{fig:purchase-model-poisoning-accuracy} show the round-wise accuracy
trajectories, while Table~\ref{tab:final-accuracy-summary} reports final test
accuracy across datasets, attacks, and aggregation rules.

The main pattern is that classical robust aggregation baselines are strongly
attack-dependent. Mean aggregation collapses under several attacks, but it is
not always the weakest baseline: under ALIE, Mean performs competitively because
the attack is constructed from coordinate-wise benign statistics and is designed
to remain within the range tolerated by simple averaging. This is consistent
with the ALIE analysis, which notes that averaging can act as a natural defense
against the specific bounded-deviation construction used by the attack. In
contrast, Median, Trimmed Mean, Krum, and Bulyan each perform well in some
settings but fail in others. For example, Krum is effective against PoisonedFL
on CIFAR-10, but collapses under Min-Max and Min-Sum on Fashion-MNIST. Bulyan is
strong under Sine on CIFAR-10 and Purchase, but is much weaker under several
distance-aware attacks. These results confirm that no single conventional
baseline is uniformly reliable across attack geometries.
CRAFT is more consistent. Across the \(18\) dataset--attack settings in
Table~\ref{tab:final-accuracy-summary}, CRAFT achieves the best final accuracy
in \(7\) settings and is within \(1.7\) percentage points of the best method in
all remaining settings. On CIFAR-10, CRAFT is best under ALIE and IPM and remains
close to the top result under Min-Max, Min-Sum, Sine, and PoisonedFL. On
Fashion-MNIST, CRAFT achieves the best result for five of the six attacks. On
Purchase, CRAFT stays close to the strongest baseline across all attacks,
showing that the footprint signal transfers beyond image CNNs to a tabular MLP
task.

DnC is the strongest competing baseline in several settings, especially on
Purchase. However, DnC requires the server to specify the number of malicious
clients, or at least a reliable upper bound on that number, before aggregation.
This is useful for comparison but is a strong deployment
assumption: underestimating the attacker count can leave malicious updates in
the aggregate, while overestimating it can discard benign signal. CRAFT does not
require this attacker-count input, yet remains competitive with DnC across the
full evaluation. 

Finally, under benign IID training, CRAFT closely tracks Mean across all three datasets, indicating that its robustness does not come at a material cost to clean accuracy (Appendix~\ref{appendix:benign-accuracy}).

\begin{figure*}[t]
\centering
\makebox[\textwidth][c]{%
\resizebox{\textwidth}{!}{%
\scalebox{1}[1.50]{%
\includegraphics{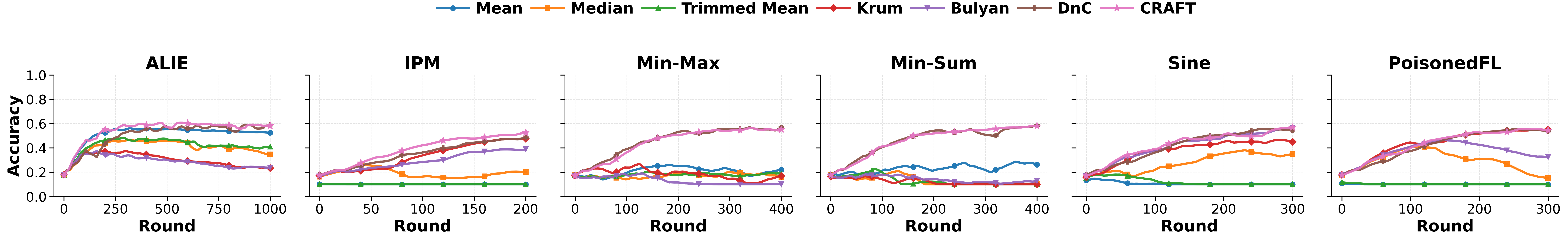}
}%
}%
}
\caption{CIFAR-10 accuracy under model-poisoning attacks.}
\label{fig:cifar10-model-poisoning-accuracy}
\end{figure*}

\begin{figure*}[t]
\centering
\makebox[\textwidth][c]{%
\resizebox{\textwidth}{!}{%
\scalebox{1}[1.50]{%
\includegraphics{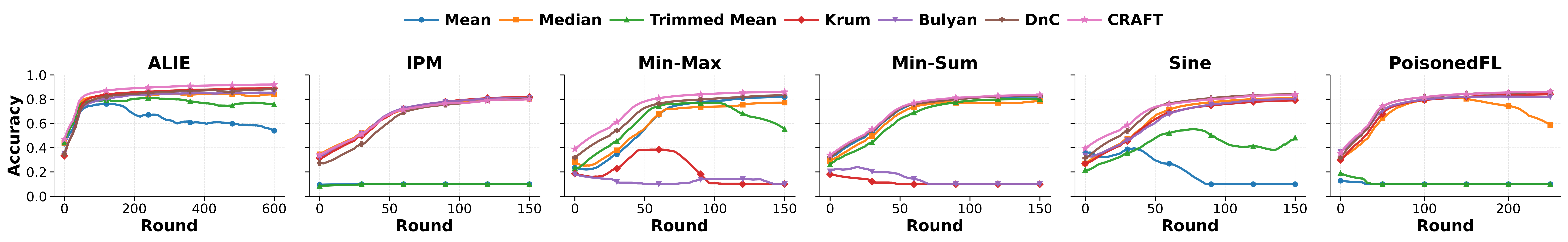}
}%
}%
}
\caption{Fashion-MNIST accuracy under model-poisoning attacks.}
\label{fig:fmnist-model-poisoning-accuracy}
\end{figure*}

\begin{figure*}[t]
\centering
\makebox[\textwidth][c]{%
\resizebox{\textwidth}{!}{%
\scalebox{1}[1.50]{%
\includegraphics{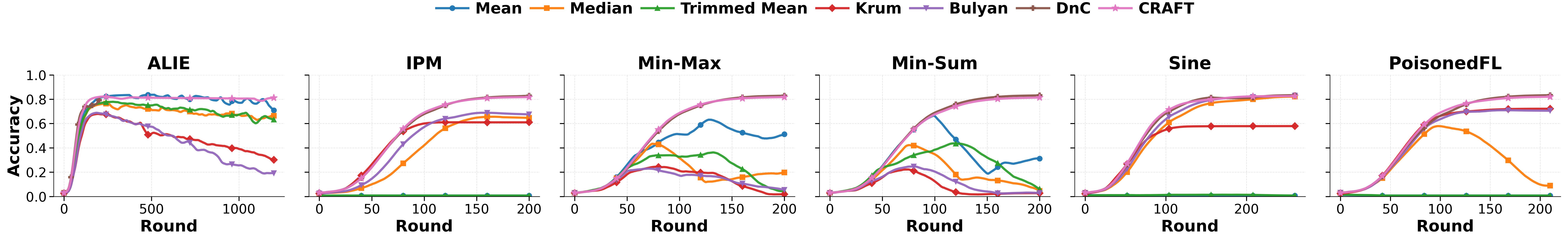}
}%
}%
}
\caption{Purchase accuracy under model-poisoning attacks.}
\label{fig:purchase-model-poisoning-accuracy}
\end{figure*}

\begin{table*}[t]
\centering
\caption{Final test accuracy (\%). The best result in each dataset--attack setting is bold.}
\label{tab:final-accuracy-summary}
\scriptsize
\renewcommand{\arraystretch}{0.95}
\setlength{\tabcolsep}{3.5pt}
\begin{tabularx}{\textwidth}{@{}ll*{6}{>{\centering\arraybackslash}X}@{}}
\toprule
\textbf{Dataset} & \textbf{Defense}
& \textbf{ALIE} & \textbf{IPM}
& \textbf{Min-Max} & \textbf{Min-Sum}
& \textbf{Sine} & \textbf{PoisonedFL} \\
\midrule

\multirow{7}{*}{\textbf{CIFAR-10}}
& Mean         & 52.30 & 10.00 & 19.37 & 23.51 & 10.00 & 10.00 \\
& Median       & 34.80 & 13.99 & 10.93 & 10.01 & 33.77 & 13.88 \\
& Trimmed Mean & 40.49 & 10.00 & 21.95 & 10.00 & 10.00 & 10.00 \\
& Krum         & 23.68 & 48.00 & 17.74 & 10.00 & 48.83 & \textbf{57.43} \\
& Bulyan       & 23.67 & 43.77 & 10.10 & 13.21 & \textbf{57.13} & 31.83 \\
& DnC          & 57.31 & 51.33 & \textbf{58.81} & \textbf{59.85} & 53.92 & 54.26 \\
\rowcolor{blue!8}\cellcolor{white} & \textbf{CRAFT}
& \textbf{58.83}
& \textbf{54.10}
& 57.78
& 58.42
& 56.09
& 56.67 \\

\midrule

\multirow{7}{*}{\textbf{Fashion-MNIST}}
& Mean         & 55.26 & 10.00 & 82.12 & 82.42 & 10.00 & 10.00 \\
& Median       & 84.23 & 80.42 & 76.28 & 78.98 & 80.60 & 47.47 \\
& Trimmed Mean & 76.59 & 10.00 & 46.54 & 80.12 & 52.19 & 10.00 \\
& Krum         & 89.09 & 82.55 & 10.00 & 10.00 & 79.63 & 84.11 \\
& Bulyan       & 85.07 & 81.22 & 10.00 & 10.00 & 81.20 & 81.53 \\
& DnC          & 88.27 & 80.29 & 83.84 & 82.51 & \textbf{84.76} & 86.16 \\
\rowcolor{blue!8}\cellcolor{white} & \textbf{CRAFT}
& \textbf{91.97}
& \textbf{82.87}
& \textbf{86.54}
& \textbf{84.24}
& 83.88
& \textbf{86.25} \\

\midrule

\multirow{7}{*}{\textbf{Purchase}}
& Mean         & 72.32 & 0.73 & 62.49 & 10.72 & 0.73 & 0.73 \\
& Median       & 65.98 & 61.47 & 22.90 & 2.47 & 82.25 & 7.97 \\
& Trimmed Mean & 66.17 & 0.73 & 4.23 & 0.73 & 0.65 & 0.73 \\
& Krum         & 30.19 & 61.00 & 2.35 & 2.72 & 57.91 & 72.11 \\
& Bulyan       & 18.70 & 67.42 & 0.73 & 3.21 & \textbf{83.48} & 70.40 \\
& DnC          & \textbf{81.76} & \textbf{83.23} & \textbf{83.16} & \textbf{83.37} & 83.46 & \textbf{83.51} \\
\rowcolor{blue!8}\cellcolor{white} & \textbf{CRAFT}
& 81.13
& 81.91
& 82.11
& 81.75
& 82.58
& 82.14 \\

\bottomrule
\end{tabularx}
\end{table*}

\subsection{Effect of Client Data Heterogeneity}
\label{subsec:data-heterogeneity}

CRAFT is designed for the IID or mildly heterogeneous regime, where honest
clients are expected to form a compact majority in compression-footprint space.
This compactness is central to the trust mechanism: if honest updates become
highly scattered because clients train on very different label distributions,
distance from the majority footprint region becomes a weaker reliability signal.
For this reason, strongly non-IID FL is outside the main theoretical scope of
CRAFT. We nevertheless include a Dirichlet label-skew sweep as a stress test.

Figure~\ref{fig:noniid-alpha-sweep} reports CIFAR-10 results under Min-Max and
Min-Sum, with lower \(\alpha\) indicating stronger heterogeneity. For Min-Max,
the expected trend is clear: as \(\alpha\) decreases from \(1.0\) to \(0.3\),
the footprint separation ratio drops from about \(7.0\) to \(4.0\), while the
attacker aggregation weight increases from \(0.28\%\) to \(1.21\%\). This
matches the intuition behind our scope: stronger label skew makes the honest
footprint region less compact, weakening the footprint signal.

Min-Sum is less monotonic. Separation falls from \(9.16\) at \(\alpha=1.0\) to
\(2.89\) at \(\alpha=0.5\), but rises again to \(5.12\) at \(\alpha=0.3\). This
does not mean that stronger non-IIDness helps CRAFT in general. Rather, stronger
label skew changes both the honest footprint dispersion and the attack geometry.
Because Min-Sum constrains the total raw-update distance to benign updates, its
solution is sensitive to the full shape of the honest-update cloud. At
\(\alpha=0.3\), the attack can remain close under this raw-distance objective
while becoming less similar to the honest majority after compression, causing
the footprint separation to partially recover.

Even under this stress test, CRAFT substantially limits attacker influence. Mean
aggregation would give the \(9\) malicious clients \(36\%\) of the aggregate
weight. CRAFT keeps attacker weight below \(1.3\%\) for Min-Max and below
\(1.9\%\) for Min-Sum across all tested \(\alpha\) values.
\begin{figure}[t]
\centering
\includegraphics[width=\columnwidth]{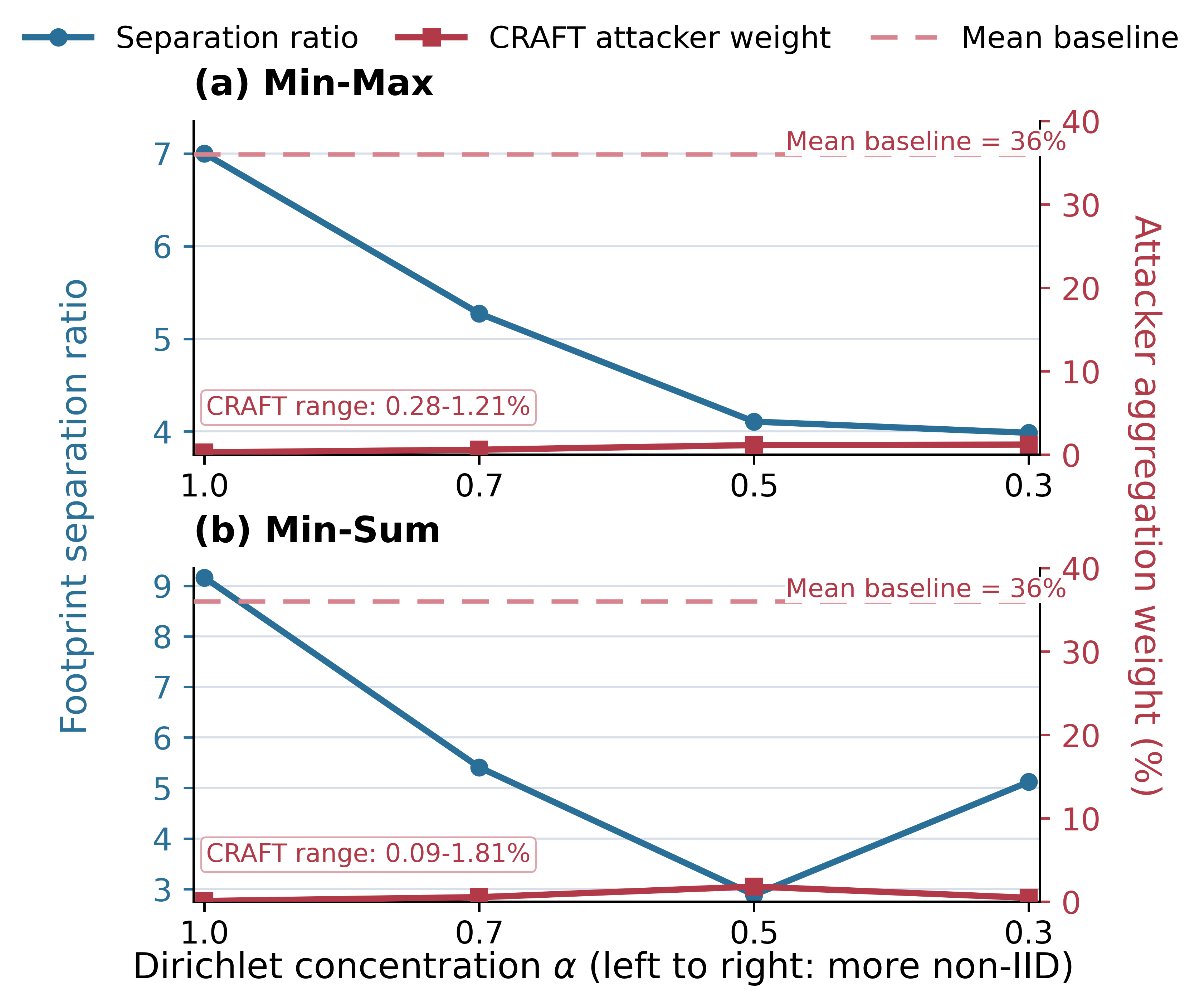}
\caption{Effect of Dirichlet label skew on CRAFT for CIFAR-10 under Min-Max and
Min-Sum attacks. Lower \(\alpha\) indicates stronger non-IIDness. The dashed
line shows the \(36\%\) attacker weight under mean aggregation.}
\label{fig:noniid-alpha-sweep}
\end{figure}
\subsection{Robustness to a Defense-Aware Attack}
\label{sec:adaptive-attack}

We next test a restricted defense-aware variant of IPM. Instead of using one
fixed attack strength, the attacker searches over a set of scalar values
\(\gamma\). For each candidate, it forms
$g_A(\gamma) = \bar{g}_H + \gamma d$,
where \(\bar{g}_H\) is a benign reference update and \(d\) is the IPM direction.
Each candidate is then evaluated after passing through the full CRAFT pipeline:
server-side recompression, footprint normalization, core selection, trust
assignment, and aggregation. The attacker selects the \(\gamma\) that maximizes
the final post-CRAFT aggregate deviation from the benign aggregate. This choice
captures the tradeoff between attack strength and trust: a large \(\gamma\) may
produce a more harmful update but receive little trust, while a moderate
\(\gamma\) may receive more trust and therefore have larger realized influence.
Thus, the selected \(\gamma\) is the scale that causes the greatest effective
damage after CRAFT has assigned trust, not necessarily the largest attack scale.
Figure~\ref{fig:adaptive-ipm-craft} shows the 50-round stress test for CIFAR-10
and Fashion-MNIST; the Purchase result is in
Appendix~\ref{appendix:adaptive-purchase}. The lower row shows the total
aggregation weight assigned to malicious clients, and the dashed line marks the
raw malicious client share, \(9/25=36\%\). On CIFAR-10, adaptive IPM closely
tracks standard IPM, while attacker weight remains mostly near zero with only
brief spikes. Fashion-MNIST is more challenging: adaptive scaling causes a
larger accuracy gap and occasional attacker-weight spikes, but the malicious
weight still remains below the raw attacker share. On Purchase, adaptive and
standard IPM are nearly indistinguishable, with attacker weight close to zero
throughout. 
\begin{figure}[t]
\centering
\includegraphics[width=0.85\columnwidth]{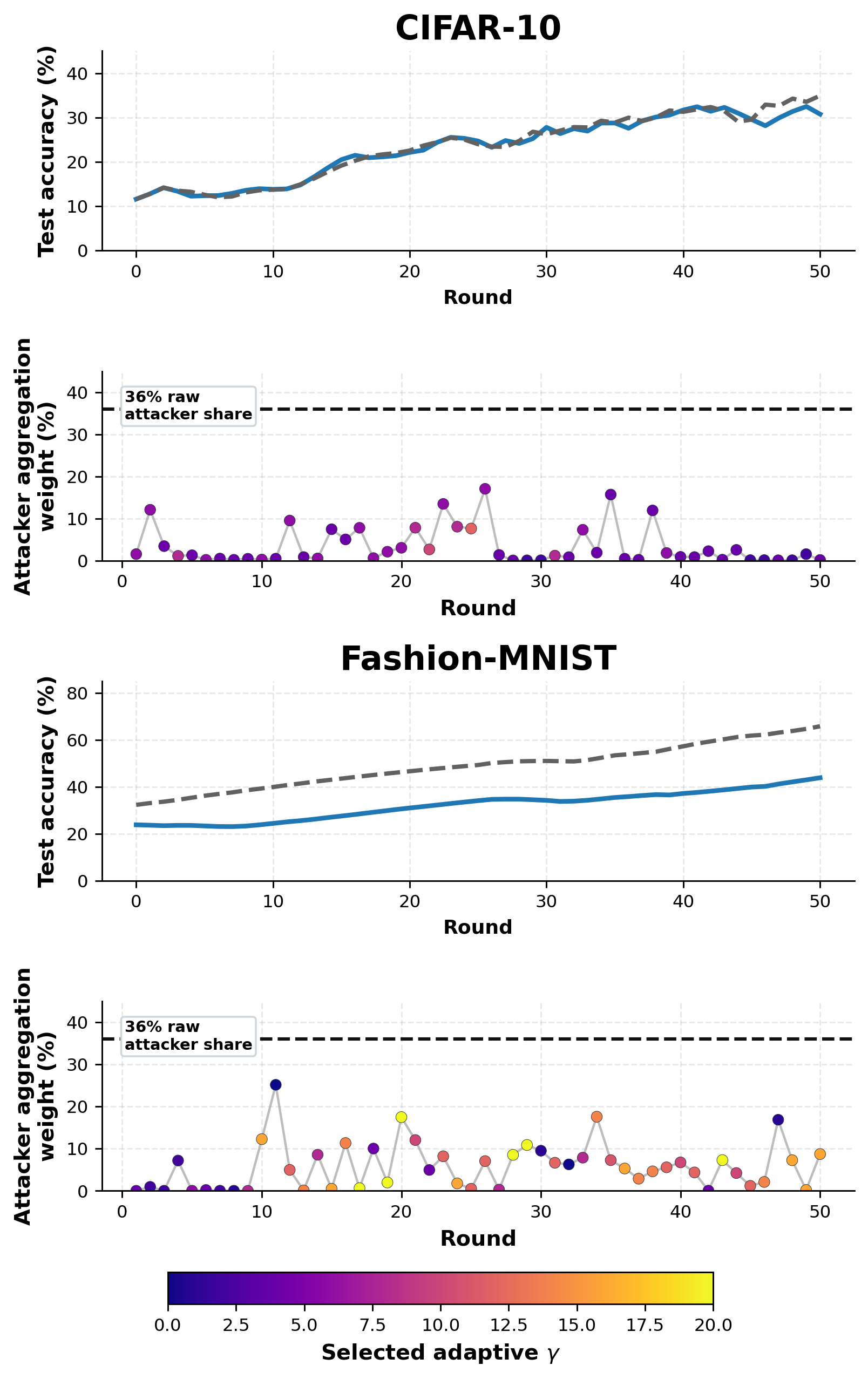}
\vspace{-0.5em}
\caption{Restricted adaptive IPM results.}
\label{fig:adaptive-ipm-craft}
\end{figure}

\subsection{Server-Side Overhead}

Across datasets, CRAFT adds server-side overhead because it performs an
additional EBLC pass to expose the compression footprint of each client update.
For CIFAR-10, the compressed-mean server path takes about 0.40 seconds per
round, while CRAFT takes about 1.35 seconds per round. For Fashion-MNIST and
Purchase, which use larger models in our experiments, CRAFT takes about 3.11
and 3.30 seconds per round, respectively.
The overhead is dominated by footprint extraction rather than the trust
computation itself. The majority-core search and trust-weight assignment take
less than 0.3 ms per round on all three datasets. Although the direct
majority-core implementation has $O(n^2 d_\phi)$ time complexity over $n$
clients, the footprint dimension is fixed at $d_\phi=5$, so this stage remains
small at the evaluated client scale. The measured overhead comes primarily from producing compression
footprints, while the low-dimensional trust computation contributes only
negligibly to the total runtime

\section{Discussion and Limitations}
\label{sec:discussion-limitations}

The effectiveness of CRAFT depends on the geometry of compression footprints,
not simply on dataset difficulty. What matters is whether honest clients form a
compact footprint majority and whether malicious updates move into a different
compression-behavior region. This explains why the signal varies across attacks,
models, and data partitions: some settings naturally expose stronger
compression-footprint separation than others.

Our mathematical argument targets the IID or mildly heterogeneous regime, where
honest footprints are expected to be coherent. Under strong non-IID partitions,
honest footprints can become more dispersed, making distance from the majority
region a weaker reliability signal. The non-IID and defense-aware experiments
therefore serve as stress tests of this assumption rather than full guarantees
outside the intended regime. CRAFT also adds server-side overhead because the
server performs diagnostic recompression to extract footprints; however, this
step is server-side, so malicious clients cannot choose or forge the footprint
measurements used for trust assignment.

\section{Conclusion}
\label{sec:conclusion}

This paper shows that EBLC-induced information loss can be used as a security
signal in federated learning. CRAFT operationalizes this idea by converting
compression-footprint statistics into trust weights before aggregation. Across
datasets and model-poisoning attacks, CRAFT is competitive with established
Byzantine-robust aggregators while requiring no prior knowledge of the exact
malicious-client count.

More broadly, this work opens a new direction for using EBLC compressors in
robust aggregation. Rather than treating lossy compression only as a
communication-efficiency mechanism, we show that the structure of its induced
loss can expose security-relevant differences between honest and malicious
updates. Future work will extend this idea to more heterogeneous client
populations, where honest clients may form multiple footprint regions rather
than a single compact majority.

\bibliographystyle{IEEEtran}
\bibliography{biblio}

\clearpage
\appendices
\section{Footprint Trust and Malicious-Client Fraction}
\label{appendix:malicious-client-fraction}

A useful trust mechanism should not only work at one fixed malicious fraction;
it should degrade predictably as the number of attackers increases. We study
this behavior on CIFAR-10 under the IPM attack by varying the malicious client
fraction while keeping the aggregation rule fixed to CRAFT.

Figure~\ref{fig:attacker-influence-majority-core} reports the effective
malicious influence that remains in the final aggregation after trust
weighting. This is the quantity that directly affects the global update: even
if malicious clients participate in a round, their effect is small when their
assigned aggregation weight is small.

CRAFT suppresses malicious influence by several orders of magnitude across a
wide range of malicious fractions. With $8\%$ malicious clients, the surviving
malicious influence is only $7.6\times 10^{-9}\%$. For $16\%$, $20\%$, and
$24\%$ malicious clients, the influence remains below $10^{-5}\%$. Even at
$32\%$ malicious clients, the effective malicious influence is only
$0.029\%$. Thus, although nearly one third of the clients are malicious, their
contribution to the final aggregated update is almost entirely suppressed.

The influence increases at higher malicious fractions, reaching $0.62\%$ at
$40\%$ malicious clients and $1.36\%$ at $44\%$. These values are still far
below the corresponding raw malicious client fractions. Under mean aggregation,
attackers would receive influence proportional to their client share; under
CRAFT, their effective contribution remains roughly one percent or less in all
tested settings.

The right panel of Figure~\ref{fig:attacker-influence-majority-core} explains
this behavior through the majority-core mechanism. CRAFT estimates a
majority-supported footprint core and assigns trust based on distance from that
core. Up to $32\%$ malicious clients, no attackers enter the majority radius,
so the core remains supported by honest clients. In this regime, attacker
footprints lie outside the dominant region and receive very small weights. At
$40\%$ and $44\%$, a small fraction of attackers enters the majority radius
($5\%$ in both cases). Once malicious clients begin to appear inside the
majority-supported region, their distance to the core decreases and their
surviving influence increases.

This behavior matches the honest-majority design assumption of CRAFT. The
experiment shows that CRAFT strongly suppresses malicious influence while the
majority core remains honest-dominated, and it also illustrates the expected
failure mode: suppression weakens when the majority-supported footprint region
begins to include malicious clients.

\begin{figure}[t]
\centering
\resizebox{\columnwidth}{!}{%
\scalebox{1}[0.75]{%
\includegraphics{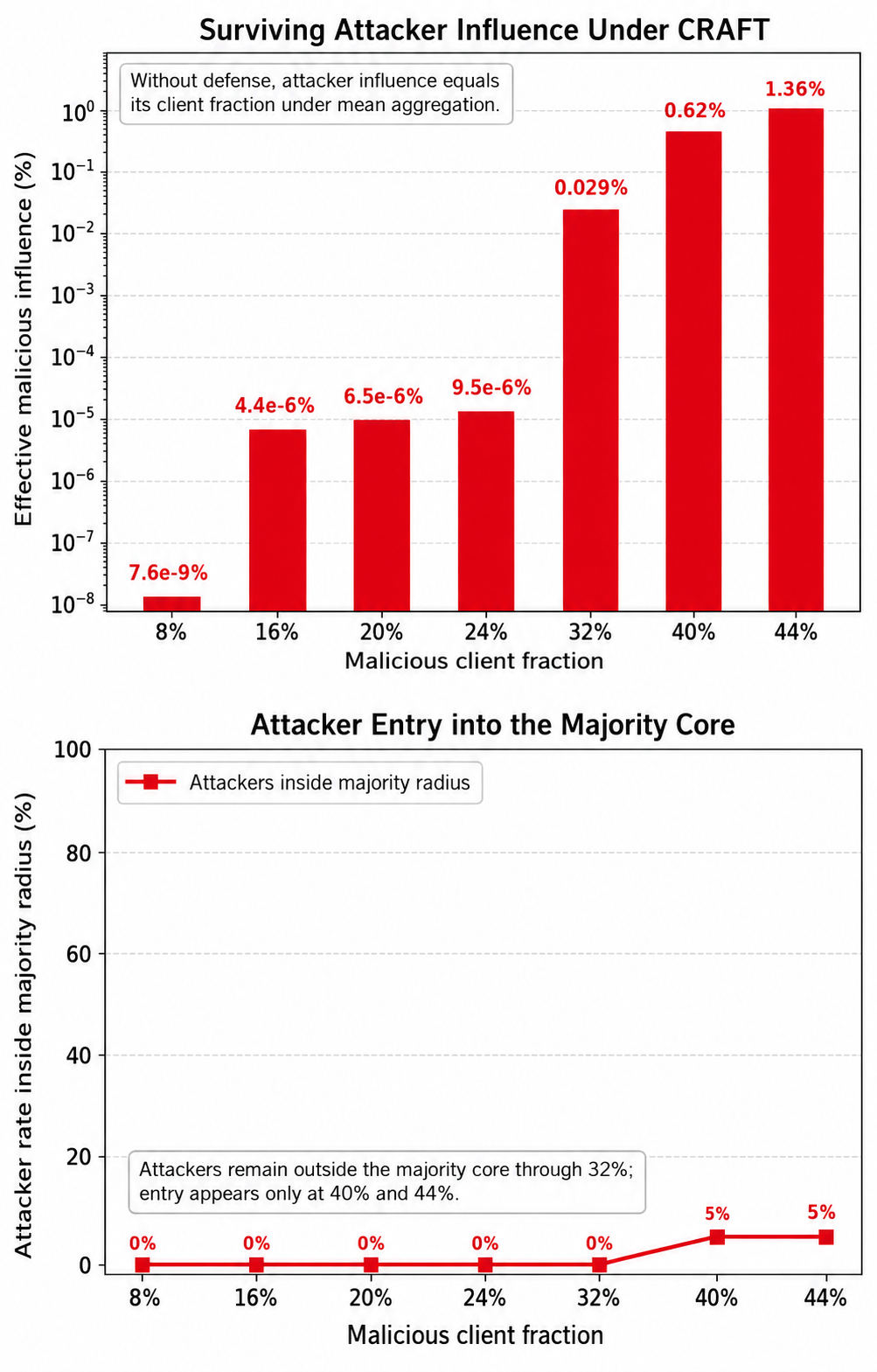}
}%
}
\caption{Effect of malicious client fraction on CRAFT for CIFAR-10 under IPM. Malicious influence remains negligible until attackers begin to enter the majority-supported footprint core.}
\label{fig:attacker-influence-majority-core}
\end{figure}

\section{EBLC and Top-\(K\) footprint distributions.}
Figure~\ref{fig:violin-topk-eblc} shows the per-metric footprint distributions
for honest and malicious clients under EBLC/SZ and Top-\(K\). The EBLC/SZ
footprints separate honest and malicious updates across several metrics, while
Top-\(K\) produces a more limited signal concentrated in fewer coordinates.

\begin{figure*}[t]
\centering
\includegraphics[width=0.95\textwidth]{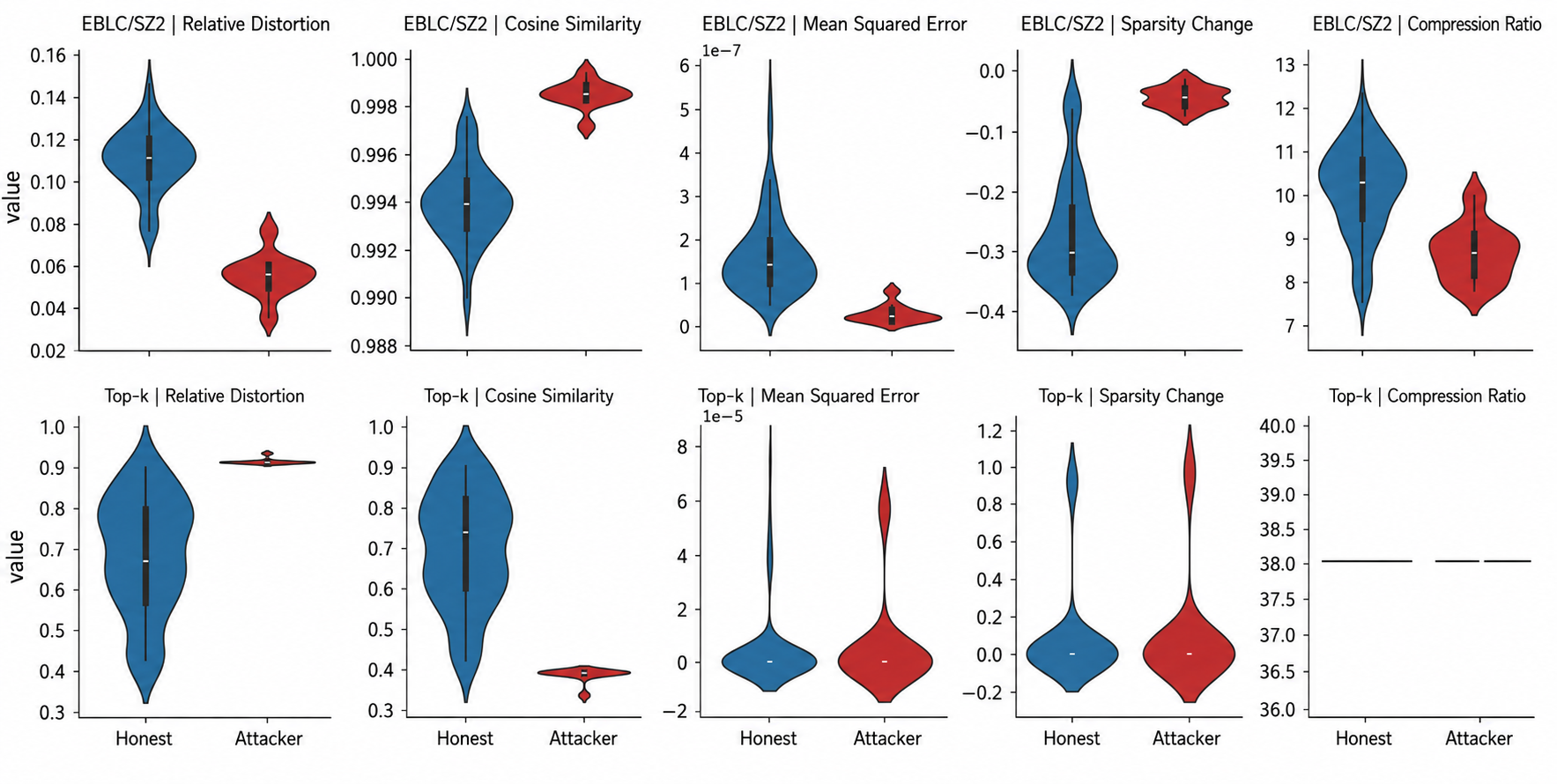}
\caption{Per-metric footprint distributions for honest and malicious clients
under EBLC/SZ (top row) and Top-\(K\) (bottom row). EBLC/SZ exhibits separation
across multiple footprint coordinates, whereas Top-\(K\) is informative mainly
in a narrower subset of metrics.}
\label{fig:violin-topk-eblc}
\end{figure*}

\section{Defense-Aware Attack on Purchase}
\label{appendix:adaptive-purchase}

Figure~\ref{fig:adaptive-ipm-craft-purchase} shows the Purchase result for the
restricted defense-aware IPM attack. Adaptive and standard IPM are nearly
indistinguishable, and CRAFT assigns almost no aggregation weight to malicious
clients across the run.

\begin{center}
\includegraphics[width=\columnwidth]{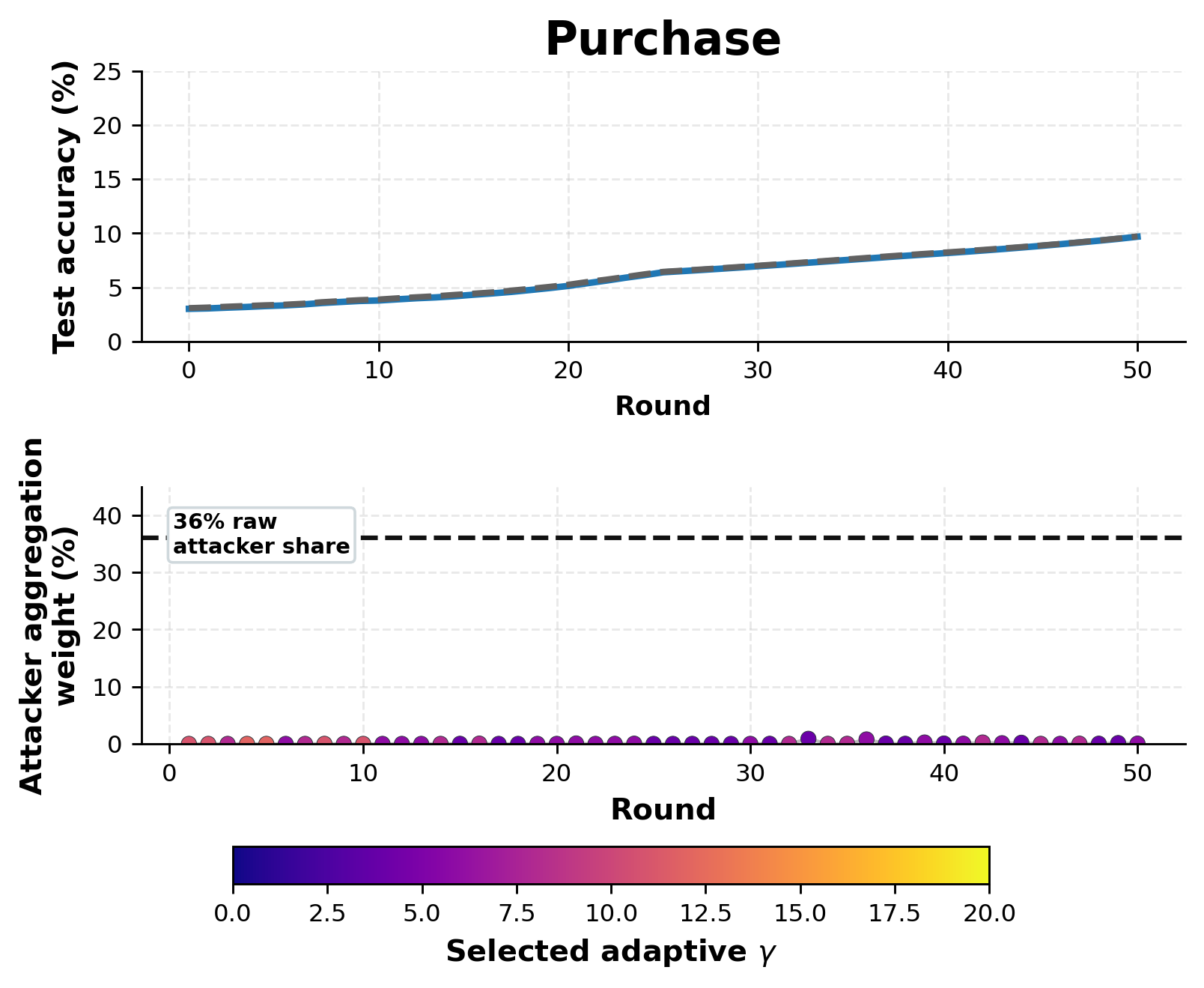}
\refstepcounter{figure}\label{fig:adaptive-ipm-craft-purchase}
{\small Fig.~\thefigure: Purchase under restricted adaptive IPM. CRAFT assigns almost no
aggregation weight to attackers.}
\end{center}

\section{Benign Accuracy Without Poisoning}
\label{appendix:benign-accuracy}

A robust aggregator should not harm training when no attack is present.
Figure~\ref{fig:no-attack-baseline} compares CRAFT with Mean aggregation under
the same compressed FedSGD pipeline.

Across CIFAR-10, Fashion-MNIST, and Purchase, CRAFT closely follows the Mean
baseline. This shows that the footprint-based trust mechanism does not
unnecessarily suppress honest updates in the benign IID setting. Therefore,
the robustness gains reported later are not caused by sacrificing clean
accuracy.

\begin{figure*}
\centering
\includegraphics[width=0.88\textwidth]{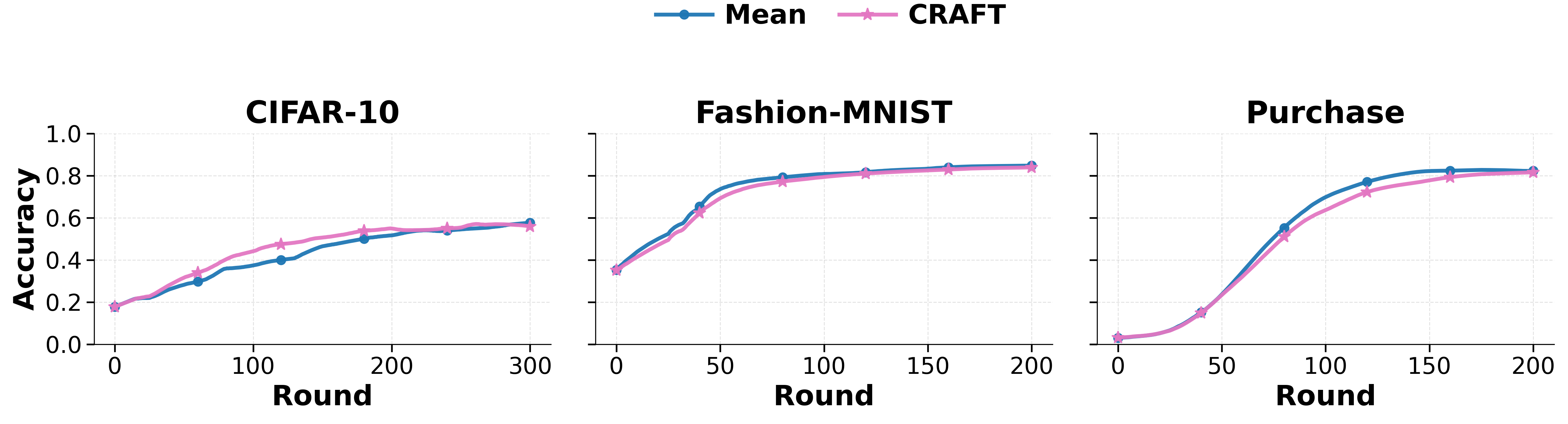}
\caption{No-attack baseline accuracy for Mean aggregation and CRAFT on
CIFAR-10, Fashion-MNIST, and Purchase. CRAFT closely tracks Mean when all
clients are honest.}
\label{fig:no-attack-baseline}
\end{figure*}
\section{Trust-Parameter Ablation}
\label{appendix:trust-parameter-ablation}

Figure~\ref{fig:trust-parameter-ablation} reports an ablation of the two
CRAFT trust-shaping parameters, \(\lambda\) and \(p\), on CIFAR-10 under IPM.
The tested settings produce similar training behavior, with the last-50-round
mean accuracy ranging from \(45.4\%\) to \(47.2\%\). We use \(\lambda=2\) and
\(p=4\) in the main experiments as a fixed operating point.

\begin{figure*}[t]
\centering
\includegraphics[width=0.90\textwidth]{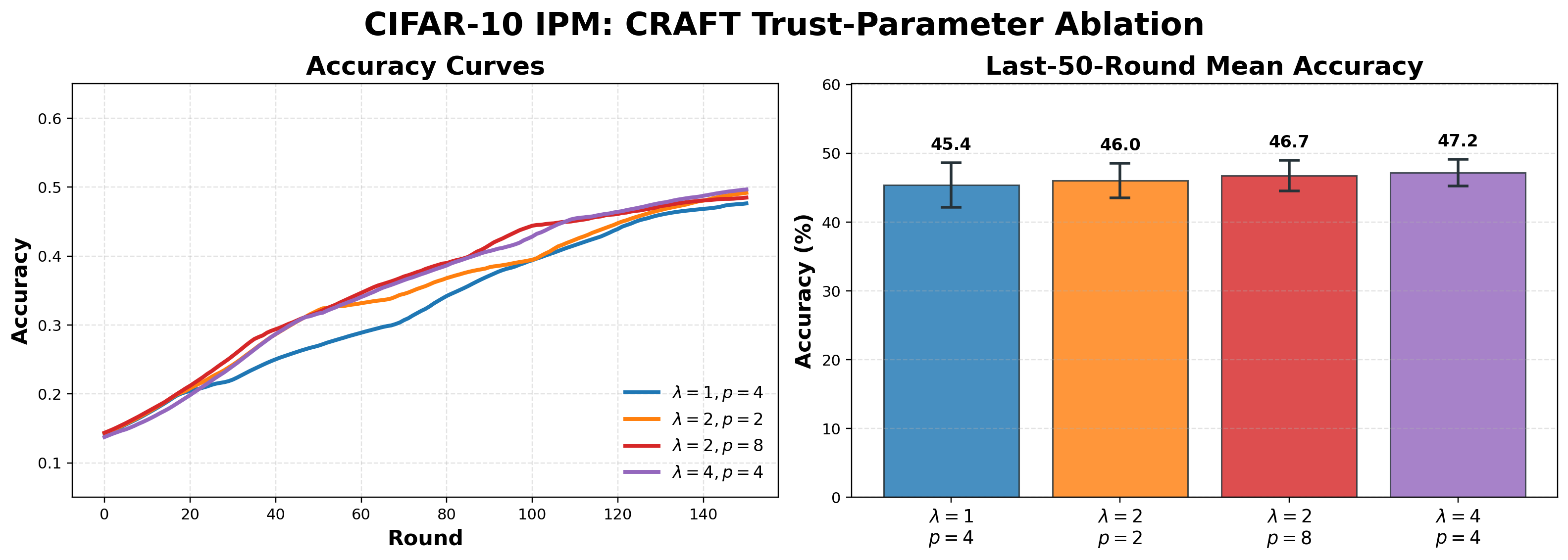}
\caption{CRAFT trust-parameter ablation on CIFAR-10 under IPM.}
\label{fig:trust-parameter-ablation}
\end{figure*}

\end{document}